%% file: main.tex
\documentclass{llncs}
\usepackage{amsmath,amssymb,amsfonts,mathtools} 
\usepackage{graphicx} 
\usepackage{tabularray}
\usepackage{diagbox}
\usepackage{arydshln}
\usepackage{tikz}
\usepackage{makecell}
\usepackage[textwidth=4.0cm,color=white,linecolor=orange]{todonotes}
\usepackage{graphicx}  
\usepackage[verbose]{wrapfig}

\usetikzlibrary{positioning, fit, calc, patterns}

\definecolor{pastelred}{rgb}{1.0, 0.41, 0.38}
\definecolor{warmblack}{rgb}{0.0, 0.26, 0.26}
\definecolor{RawSienna}{rgb}{0.83, 0.54, 0.35}
\definecolor{pastelgreen}{rgb}{0.47, 0.87, 0.47}

\input{utils/include}
\input{utils/general_utils}

\input{utils/math_utils}

\title{Navigating the Deep: \\ End-to-End Extraction on Deep Neural Networks}

\author{}
\institute{}
\author{%
  Haolin Liu\inst{1,2} \and
  Adrien Siproudhis\inst{2} \and
  Samuel Experton\inst{3} \and
  Peter Lorenz\inst{2} \and \\
  Christina Boura\inst{3,4} \and
  Thomas Peyrin\inst{2}
}
\institute{%
  Shanghai Jiao Tong University, China\inst{1} \\
  Nanyang Technological University, Singapore\inst{2} \\
  IRIF, Université Paris Cité, France\inst{3} \\
  Institut universitaire de France (IUF), Paris, France\inst{4}
}

\begin{document}
\maketitle

\begin{abstract}
Neural network model extraction has recently emerged as an important security concern, as adversaries attempt to recover a network’s parameters via black-box queries. Carlini et al. proposed in CRYPTO’20 a model extraction approach inspired by differential cryptanalysis, consisting of two steps: signature extraction, which extracts the absolute values of network weights layer by layer, and sign extraction, which determines the signs of these signatures. However, in practice this signature-extraction method is limited to very shallow networks only, and the proposed sign-extraction method is exponential in time. Recently, Canales-Martínez et al. (Eurocrypt’24) proposed a polynomial-time sign-extraction method, but it assumes the corresponding signatures have already been successfully extracted and can fail on so-called low-confidence neurons.

In this work, we first revisit and refine the signature extraction process by systematically identifying and addressing for the first time critical limitations of Carlini et al.'s signature-extraction method. These limitations include rank deficiency and noise propagation from deeper layers. To overcome these challenges, we propose efficient algorithmic solutions for each of the identified issues, greatly improving the capabilities of signature extraction. Our approach permits the extraction of much deeper networks than previously possible. 

In addition, we propose new methods to improve numerical precision in signature extraction, and enhance the sign extraction part by combining two polynomial methods to avoid exponential exhaustive search in the case of low-confidence neurons. This leads to the very first end-to-end model extraction method that runs in polynomial time.  

We validate our attack through extensive experiments on ReLU-based neural networks, demonstrating significant improvements in extraction depth. For instance, our attack extracts consistently at least eight layers of neural networks trained on either the {\tt MNIST} or {\tt CIFAR-10} datasets, while previous works could barely extract the first three layers of networks of similar width. Our results represent a crucial step toward practical attacks on larger and more complex neural network architectures. \end{abstract}


\input{content/intro}

\input{content/preliminaries}

\input{content/strategies}
\input{content/conclusion}
\bigskip

\noindent \textbf{Acknowledgements.} The authors with NTU affiliation are supported by the Singapore NRF-NRFI08-2022-0013 grant. Christina Boura is partially supported by the France 2030 program under grant agreement No. ANR-22-PECY-0010 Cryptanalyse. Haolin Liu is also supported by the National Natural Science Foundation of China (No. 62372294) and the China Scholarship Council (No. 202406230017).

\bibliographystyle{splncs04}

\appendix

\input{content/appendix}

\end{document}

%% file: utils/include.tex
\usepackage{placeins}
\usepackage{hyperref} 
\usepackage{algorithm}
\usepackage{algorithmic}
\usepackage{float} 
\usepackage{caption}  

\usepackage{microtype}
\usepackage{booktabs} 

\usepackage[capitalize,noabbrev]{cleveref}
\usepackage{pifont}
\usepackage{threeparttable}
\usepackage{multirow}
\usepackage{url}

\usepackage{enumitem}
\usepackage{xspace}
\usepackage{xfrac}
\usepackage{comment}
\usepackage{tikz}
\usepackage{etoolbox}

\usepackage{subcaption}
\usepackage{environ}

\usepackage{pgfplots}
\usetikzlibrary{patterns}
\pgfplotsset{compat=1.18}

%% file: utils/general_utils.tex
\usepackage{color, colortbl}
\definecolor{Gray}{gray}{0.93}
\definecolor{Orange}{rgb}{1,0.5,0}
\definecolor{DGray}{gray}{0.83}
\definecolor{LightCyan}{rgb}{0.88,1,1}
\definecolor{asparagus}{rgb}{0.53, 0.66, 0.42}

\usepackage[most]{tcolorbox}

%% file: utils/math_utils.tex
\usepackage{bm}

\def\eqref#1{(\ref{#1})}

\def\1{\bm{1}}

\DeclareMathAlphabet{\mathsfit}{\encodingdefault}{\sfdefault}{m}{sl}
\SetMathAlphabet{\mathsfit}{bold}{\encodingdefault}{\sfdefault}{bx}{n}



%% file: content/intro.tex
\section{Introduction}

Neural Networks (NNs) are a class of machine learning models composed of layers of interconnected units (or neurons) that transform input data and learn to recognize patterns through a training process. Deep Neural Networks (DNNs) are a subset of NNs with multiple hidden layers, enabling them to model highly complex functions. DNNs have become indispensable in various applications, including computer vision (e.g., image recognition, and video analysis), natural language processing, automated medical diagnostics, and fraud detection, among others. However, training DNNs requires large datasets, significant computational resources, and carefully fine-tuned algorithms~\cite{BishopBook2025,Xiao2025FoundationsOL}. As a result, trained DNNs have become valuable intellectual assets, making them attractive targets for attackers seeking to extract the model rather than invest in training their own. 

In many cases, DNNs are deployed as cloud-based or online services, allowing users to interact with them without direct access to their internal parameters~\cite{googlecloud,huggingface}. A natural question, therefore, is how well different DNNs resist attacks in this black-box setting, where an adversary can query the network using random or carefully chosen inputs, potentially adaptively, and exploit the observed outputs to reconstruct either the exact internal parameters or an approximation sufficient to construct a functionally equivalent network. 

Neural network model extraction is an old and well-studied problem. 
The earliest work dates back to 1994, when Fefferman~\cite{fefferman1994reconstructing} proved that the output of a sigmoid network uniquely determines its architecture and weights, up to trivial equivalences. Later, in 2005, Lowd and Meek~\cite{lowd2005adversarial} introduced new algorithms for reverse engineering linear classifiers and applied these techniques to spam filtering, marking one of the first practical extraction attacks.
The field saw renewed interest in 2016 with the work of Tram\`er et al.~\cite{USENIX:TZJRR16}, who demonstrated attacks on deployed machine learning models accessible through APIs that output high-precision confidence values (also referred to as scores), in addition to class labels. 
By exploiting these confidence scores, they successfully mounted attacks on various model types, including logistic regression  and decision trees.

For non-linear models,~\cite{USENIX:TZJRR16} and later~\cite{Papernot2016PracticalBA} introduced the notion of \textit{task accuracy extraction}, where the goal is to build a model that performs well on the same decision task, achieving high accuracy on predictions, without necessarily reproducing the exact outputs of the original model. This is different from \textit{functionally equivalent extraction}, where the objective is to replicate the original model’s outputs on all inputs, regardless of the ground truth. 
Jagielski et al.~\cite{jagielski2020high}, who introduced this taxonomy, argued that learning-based approaches like~\cite{Papernot2016PracticalBA,USENIX:TZJRR16} are fundamentally limited when it comes to achieving functionally equivalent extraction.
Early attempts for functionally equivalent extraction relied on access to internal gradients of the model~\cite{milli2019model}, side-channel information~\cite{batina2019csi}, or were limited to extracting only a small number of layers~\cite{jagielski2020high}. 

A significant step in model extraction came in 2020 with Carlini et al.~\cite{C:CarJagMir20}, who studied ReLU networks in the S5 setting of~\cite{cryptoeprint:2024/1580}, which exposes the network’s pre-normalization outputs. In this work, the authors reframed the problem from a cryptanalytic perspective by observing that DNNs share key similarities with block ciphers, including their iterative structure, alternating linear and non-linear layers, and the use in parallel of a small function (S-box in block ciphers, activation function in neural networks) to ensure non-linearity. Taking inspiration from cryptanalytic attacks on block ciphers, particularly differential cryptanalysis~\cite{C:BihSha90}, Carlini et al. proposed a new two-step iterative approach for recovering a model’s internal parameters. Their method decomposes the extraction process into two distinct phases: \emph{signature extraction} and \emph{sign extraction}. In the first step, the absolute values of the weights in a layer are reconstructed (signature extraction), followed by a second step where the correct signs are determined (sign extraction). This process is applied iteratively, proceeding layer by layer. This approach enabled them to successfully extract neural networks, including a model with $100{,}000$ parameters trained on the {\tt MNIST} digit recognition task, using only $2^{21.5}$ queries and less than an hour of computation. However, their method for extracting the signs had an exponential complexity in the number of neurons, limiting its feasibility to small networks with typically very few (at most 3) layers.

In 2024, Canales-Martínez et al.~\cite{EC:CCHRSS24} extended the work of Carlini et al.~\cite{C:CarJagMir20} by introducing several new algorithms for the sign-extraction step in ReLU-based neural networks. Their improved sign-extraction methods permitted them to extract the parameters of significantly larger networks than those considered in~\cite{C:CarJagMir20}. For instance, they reported successfully extracting the parameters of a neural network with 8 hidden layers of 256 neurons each (amounting to 1.2 million parameters) trained on the {\tt CIFAR-10} dataset~\cite{krizhevsky2009cifar} for image classification across 10 categories. It is important to note, however, that their experiments assumed the absolute values of the weights had already been recovered using the method of Carlini et al.~\cite{C:CarJagMir20} and were not implemented as full end-to-end attacks.

More recently, Foerster et al.~\cite{Hannah} conducted a detailed analysis of the signature and sign extraction procedures, implementing and benchmarking them on various ReLU-based neural networks. Their work revealed that, contrary to what was previously believed, the main practical bottleneck preventing these attacks from advancing beyond the early layers lies in the signature-extraction step introduced in~\cite{C:CarJagMir20}. Because of this, both~\cite{C:CarJagMir20} and~\cite{Hannah} failed in practice to recover weight parameters beyond the third layer. Although the authors of~\cite{EC:CCHRSS24} reported successful extraction from deeper networks, this was assuming that signatures were already extracted with~\cite{C:CarJagMir20}'s method, which could not work on deeper parts of the networks. In addition, Foerster et al.~\cite{Hannah} also identified a critical limitation in the sign-extraction method of~\cite{EC:CCHRSS24}: for many networks, the confidence in the recovered signs of neurons does not increase even with additional iterations, contradicting the claim in~\cite{EC:CCHRSS24} that wrong signs could be fixed by testing more critical points. To address this, they proposed applying exhaustive search, similar to~\cite{C:CarJagMir20}, for neurons with low confidence. This indicates that, for many networks, sign extraction has not yet reached true polynomial-time extraction, leaving it as another unresolved bottleneck.

In parallel with these efforts to improve performance under full access to confidence scores, other members of the research community have turned their attention to the more realistic ``hard-label'' scenario~\cite{cryptoeprint:2024/1580,AC:CDGSWW24,canales2025extracting}, (S1 setting in the taxonomy of~\cite{cryptoeprint:2024/1580}) where the DNN outputs only the predicted class label, corresponding to the highest confidence score, while the actual confidence values remain hidden. We show however that important problems persist even in the more informative S5 setting, and resolving them is a prerequisite for full extraction of deep networks. Moreover, because many techniques are shared across settings, advances achieved in S5 naturally carry over and improve extraction in the more restrictive scenarios as well. 

To date, no method can efficiently recover a network’s parameters beyond the third hidden layer in any setting (from S1 to S5)\footnote{The authors of~\cite{cryptoeprint:2024/1580} report having successfully extracted the 4th layer of a single network. However, in that network, this 4th layer is very contractive, causing the issues identified in our work not to arise.}. This highlights the need for significant improvements in the extraction to enable attacks on deeper networks, as well as an exploration of more challenging end-to-end attacks (attacks recovering a model’s parameters solely from black-box input–output pairs, without any simplification such as artificially separating sign and signature extraction).

 \subsection*{Our Contributions}

Prior attacks do not scale to deep neural networks not merely because increased depth raises computational cost, but because their methods are incomplete and overlook issues that grow more severe in deeper models. In this work, we diagnose these challenges and propose practical solutions, thereby enabling effective end-to-end extraction on much deeper neural networks.

\medskip
\noindent{\bf Reaching deeper layers.} By identifying and addressing two critical limitations in signature extraction which arise as network depth increases, we allow signature extraction to progress beyond the first few layers:

\begin{itemize}
\item {\it Incorrect signature extraction due to rank-deficient systems.}
  The original attack produces incorrect signatures at deeper layers when the linear system used to compute signatures is rank-deficient. We provide a method to increase the rank, yielding correct signature extraction on deeper networks.
    \item {\it Overlooked influence of deeper layers.}
    Contrary to what was previously believed, we show that signatures originating from deeper layers than the one attacked can be misinterpreted as coming from the target layer. We propose strategies to filter them out depending on the neural network targeted.
\end{itemize}

\medskip

\noindent {\bf End-to-end extraction.} We present the first end-to-end model extraction that runs in polynomial time, achieved through notable enhancements to signature precision improvement and sign-extraction methods.
\begin{itemize}
    \item {\it Signature precision improvement.} In addition to Carlini et al.~\cite{C:CarJagMir20}'s approach for improving signature precision, we identify and address two additional sources of imprecision in signature extraction, which improves extraction efficiency and supports wide layers.
    \item {\it Confident sign extraction.} As noted by~\cite{Hannah}, the neuron wiggle method of~\cite{EC:CCHRSS24} fails for the networks we consider on so-called low-confidence neurons. Instead of the exponential exhaustive search that~\cite{Hannah} uses, we combine two polynomial methods from~\cite{EC:CCHRSS24} to avoid relying on low-confidence neurons.
\end{itemize}

We validate all our improvements through extensive experiments on ReLU-based neural networks, showing significant improvements in extraction depth. Notably, we present the first end-to-end model extraction on eight-layer networks with eight neurons per layer 
trained on MNIST and CIFAR-10 datasets, extracting almost all the weights and achieving a very low relative error $\epsilon \leq 3.6 \times 10^{-4}$ on $73\%$ of the input space,
while previous works could barely extract the first three layers. These results mark a significant step toward practical attacks on larger, more complex neural network architectures. Our full code is accessible in 
\url{https://github.com/PsyduckLiu/End-to-End-Deep-Neural-Network-Extraction}.

The rest of the article is organized as follows. Section~\ref{sec:preliminaries} introduces preliminaries: notations, the attack model, and a short overview of the original signature-extraction and sign-extraction methods. Section~\ref{sec:strategies} presents our improvements to signature extraction. Section~\ref{sec:signs} describes our new sign-extraction strategy. Section~\ref{sec:eval_weights} introduces our evaluation metrics and discusses a taxonomy of unrecovered weights. Section~\ref{experiments} reports our experimental evaluation and end-to-end results. We evaluate our attack on three different networks: one trained on {\tt CIFAR-10} with architecture $3072\text{-}8^{(8)}\text{-}1$ and two networks trained on {\tt MNIST} with architectures $784\text{-}8^{(8)}\text{-}1$ and $784\text{-}16^{(8)}\text{-}1$, which we refer to as Model I, II, and III, respectively. Finally, Section \ref{sec:conclusion} concludes and discusses limitations.

%% file: content/preliminaries.tex
\section{Preliminaries} \label{sec:preliminaries}
In this section, we introduce important definitions and notations (\Cref{definitions}), followed by assumptions regarding the attack setting and goal (\Cref{assumptions}) and finally an overview of the original attack of~\cite{C:CarJagMir20} and~\cite{EC:CCHRSS24} (\Cref{overview}).

\subsection{Notations and Definitions}
\label{definitions}
This paper models neural networks as parametrized functions whose parameters are the unknowns we aim to extract. Our results do not depend on how neural networks are trained or applied. As a result, no prior knowledge about them is required to understand the attack.

A neural network consists of fundamental units called neurons, which are organized into layers and connected to other neurons in both the previous and next layers. Each neuron has an associated weight vector for its incoming neurons from the previous layer, along with a bias term that influences its output. Following the notations and definitions from~\cite{EC:CCHRSS24,C:CarJagMir20}, we now present several central definitions related to neural networks. A simple example to illustrate the definitions can be found in Appendix~\ref{example_network}.

\begin{definition}[$r$-deep neural network]
    \label{def:dnn}
    An \emph{$r$-deep fully connected neural network} of architecture $[d_0, \cdots, d_{r+1}]$ is a function $f: \mathbb{R}^{d_0} \to \mathbb{R}^{d_{r+1}}$ composed of alternating linear layers $\ell^{(i)}: \mathbb{R}^{d_{i-1}} \to \mathbb{R}^{d_{i}}$ and a non-linear activation function $\sigma$ acting component-wise such that: $
     f = \ell^{(r+1)} \circ \sigma \circ \dots \circ \sigma \circ \ell^{(2)} \circ \sigma \circ \ell^{(1)}$, where $\ell^{(i)}(x) = A^{(i)}(x)+b^{(i)}$; $A^{(i)} \in \mathbb{R}^{{d_i\times d_{i-1}}}$,  $b^{(i)} \in \mathbb{R}^{d_i}$. 
  We call:
    \begin{itemize}[leftmargin=0.5cm]
        \item $r$: the number of layers (or \emph{depth} of the network).
        \item $d_i$: the number of neurons in the $i$-th layer (or \emph{width} of layer $i$).
        \item $\ell^{(i)}$: the $i$-th linear layer function.
        \item $A^{(i)}$: the $i$-th linear layer weight matrix.
        \item $b^{(i)}$: the $i$-th linear layer bias vector.
    \end{itemize}
\end{definition}

To extend the notations, the architecture $[d_0, \dots, d_{r+1}]$ can also be written as $d_0-\dots-d_{r+1}$. For consecutive layers with the same dimension, an exponential notation can be used for compactness, e.g., $20-10^{(3)}-1$ represents the architecture $20-10-10-10-1$.

As in~\cite{EC:CCHRSS24,C:CarJagMir20,Hannah}, this research only considers fully connected neural networks using the widespread ${\rm ReLU}$ activation function~\cite{Nair2010RectifiedLU} applied component-wise. The structure of a fully connected neural network resembles that of substitution-permutation networks (SPNs) such as the AES~\cite{aes}. The component-wise application of the activation function is analogous to the use of S-boxes in SPNs, hence the analogy with cryptanalysis introduced in \cite{C:CarJagMir20}.

\begin{definition}[ReLU neural network] \label{def:reludnn} We say that a neural network $f$ defined as above is a \emph{ReLU neural network} if its non-linear activation function $\sigma$ is 
 \begin{align*}
 \sigma(v) &= ({\rm ReLU}(v_1), {\rm ReLU}(v_2), \dots, {\rm ReLU}(v_n))\\
 &= (  \max(v_1,0), \max(v_2, 0), \dots, \max(v_n, 0))
 \end{align*} for $v = (v_1, v_2, \dots, v_n) \in \mathbb{R}^n$.
\end{definition}

In the following, we define a reduced-round neural network $F^{(i)}$ for $1\leq i \leq r$ as a function $F^{(i)}:\mathbb{R}^{d_{0}} \to \mathbb{R}^{d_{i}}$ given by:  
\[
    F^{(i)} = \sigma \circ \ell^{(i)} \circ \sigma \circ \dots \circ \sigma \circ \ell^{(2)} \circ \sigma \circ \ell^{(1)}.
\]

$F^{(i)}$ shares the same linear transformations and activation functions as the original function $f$, up to layer $i$. It will be used to track the transformation of an input vector as it goes through $f$. By $F^{(r+1)}$ and $F^{(0)}$, we mean respectively $f$ and the identity function.

The \emph{$k$-th neuron} of layer $i$ is defined as the function $\eta^{(i)}_k(x) = A^{(i)}_k \cdot x + b^{(i)}_k$, where $A^{(i)}_k$ is the $k$-th row of $A^{(i)}$ and $b^{(i)}_k$ is the $k$-th coordinate of $b^{(i)}$.

A central notion in the analysis of~\cite{C:CarJagMir20} is that of \emph{critical points} of neurons. They enable the extraction of their corresponding neuron’s weights through carefully chosen queries.

\begin{definition}[critical point]
    An input point $x \in \mathbb{R}^{d_0}$ is called a \emph{critical point} of a neuron $\eta_k^{(i)}$ if
    \[
    \eta_k^{(i)}\circ F^{(i-1)}(x)=0
    \] 
\end{definition} 
The activation pattern for a given input $x$ is the set of all (active) neurons that contribute to the output $f(x)$.

\begin{definition}[activation pattern] The \emph{activation pattern} of $x$ through $f$ at layer $i$, for $1\leq i \leq r$, is the set 
\[\mathcal{S}^{(i)}_f(x) := \{\eta_k^{(j)} \: \big| \: \eta_k^{(j)} \circ F^{(j-1)}(x)> 0, \quad 1\leq j \leq i, 1\leq k \leq d_j\}.\]
If $i = r$, we note $\mathcal{S}^{(i)}_f$ as $\mathcal{S}_f$ and refer to it as the activation pattern of $x$ through $f$.
If $\eta_k^{(i)} \in \mathcal{S}_f(x)$ then we say that neuron $\eta_k^{(i)}$ is \emph{active}, otherwise we say that it is \emph{inactive} (for $x$ and $f$).
\end{definition}

\begin{definition}[polytope]
    The polytope of $x$ at layer $i$, $\mathcal{P}^{(i)}_x$, is the largest connected open subset of
    \[\{x'\in\mathbb{R}^{d_0} \: | \: \mathcal{S}^{(i)}_f(x)=\mathcal{S}^{(i)}_f(x')\}
    \]
    containing $x$.
    If $i = r$, we note $\mathcal{P}^{(i)}_x$ as $\mathcal{P}_x$ and refer to it as the polytope of $x$.\\
\end{definition}
Therefore, $\mathcal{P}_x$ represents the region of the input space consisting of all points $x'$ connected to $x$ that have the same activation pattern. Let's now justify the use of the term \emph{polytope} to describe this space.

\begin{lemma}[local affine network]
First, for $x \in \mathbb{R}^{d_0}$, for all $1 \leq i \leq r$, the network $F^{(i)}$ is affine on $\mathcal{P}^{(i)}_x$, meaning that there exists $\Gamma_x^{(i)} \in \mathbb{R}^{d_{i}\times d_0}$, $\gamma^{(i)}_x\in \mathbb{R}^{d_{i}}$ such that $\forall x' \in \mathcal{P}^{(i)}_x$, we have $F^{(i)}(x') = \Gamma^{(i)}_x x' + \gamma^{(i)}_x.$ We note:
\begin{itemize}
    \item $F^{(i)}_x$ the function $x' \mapsto \Gamma_x^{(i)} x' + \gamma_x^{(i)}$.
    \item $f_x$ and $F_x^{(r+1)}$ the function $x' \mapsto\ell^{(r+1)}\circ F^{(r)}_x(x')$.
\end{itemize}

Second, it follows that if $x$ is not a critical point, there exists \( \epsilon > 0 \) such that for all \( x' \in \mathcal{B}(x; \epsilon) \),
\begin{equation*}
    f(x') = f_x(x') \quad \text{and} \quad \forall i \in \{1, \dots, r+1\}, \quad F^{(i)}(x') = F^{(i)}_x(x'),
\end{equation*}
where $\mathcal{B}(x; \epsilon) \subset \mathbb{R}^{d_0}$ is the ball of radius $\epsilon$ centered at $x$.
\end{lemma}
Though ReLUs are not linear, they are piecewise linear. This means that when moving near a non-critical point, all ReLUs act as linear functions. 
Recall that critical points are solutions of an equation of the form $A^{(i)}_k \cdot  F^{(i-1)}(x) + b^{(i)}_k= 0$. 
Thus, they form hyperplanes which partition the input space into these different affine regions, hence referred to as \emph{polytopes}.

We denote by $I_x^{(i)}$ the diagonal matrix representing which neurons in layer $i$ are active for the input $x$: the $k$-th coefficient of the diagonal is $1$ if the $k$-th neuron in layer $i$ is active, and $0$ otherwise. Thus,
\begin{align*}
F^{(i)}(x) &=   I_x^{(i)}(A^{(i)} \cdots (I_x^{(2)}(A^{(2)}(I_x^{(1)}(A^{(1)}x + b^{(1)})) + b^{(2)}) \cdots + b^{(i)}) \\
&= I_x^{(i)}A^{(i)} \cdots I_x^{(2)}A^{(2)}I_x^{(1)}A^{(1)}x + \gamma^{(i)}_x = \Gamma^{(i)}_x x + \gamma^{(i)}_x.
\end{align*}
If $x$ and $x'$ are two points in the same polytope, then $\forall i, I^{(i)}_x = I^{(i)}_{x'}$ implying that $\Gamma_x^{(i)} = \Gamma_{x'}^{(i)}$ and $ \gamma_x^{(i)} = \gamma_{x'}^{(i)}$, i.e., that $F^{(i)}_x = F^{(i)}_{x'}$. 

\subsection{Adversarial Resources and Goal}
\label{assumptions}
We consider two parties in this model extraction attack: an oracle $\mathcal{O}$ and an adversary. The adversary generates queries $x$ and sends them to the oracle, which then responds with the correct output $f(x)$.

\medskip
\noindent \textbf{Adversarial resources.} We make the following assumptions regarding the target neural network and the attacker's capabilities. These are the same assumptions as in~\cite{EC:CCHRSS24,C:CarJagMir20,Hannah}:
\begin{itemize}[leftmargin=0.5cm]
\item \textbf{Fully connected ReLU network}. $f$ is a fully connected ReLU neural network.
    \item \textbf{Known architecture}. The attacker knows the architecture of the target neural network $f$.
    \item \textbf{Unrestricted input access}. The attacker can query the network on any input $x\in \mathbb{R}^{d_0}$.
    \item \textbf{Raw output access}. The oracle returns the complete raw output $f(x)$, with no post-processing.
    \item \textbf{Precise computations}. The oracle computes $f(x)$ using 64-bit arithmetic.
\end{itemize}

\smallskip
\noindent \textbf{Adversarial goal.}
The objective of the extraction is not to exactly replicate the target network, but rather to achieve what is known as an \emph{$(\epsilon, \delta)$-functionally equivalent extraction}~\cite{C:CarJagMir20}. We slightly change the original definition to normalise the difference between the target and the extracted network.
\begin{definition}[functionally equivalent extraction] \label{def:fe}
    We say that two models $f$ and $\hat{f}$ are {\normalfont $(\epsilon, \delta)$-functionally equivalent} on an input space $S \subset \mathbb{R}^{d_0}$ if $\forall x \in S, \> \mathbb{P}(|\frac{\hat{f}(x)-f(x)}{f(x)}| \leq \epsilon) \geq 1-\delta$.
\end{definition} 
\subsection{Overview of the Existing Signature Extraction}
\label{overview}
In this section, we recall the signature-extraction algorithm from~\cite{C:CarJagMir20}. 
The extraction is performed layer by layer: first, we recover the parameters of layer 1, then use them to reconstruct layer 2, and so on. We now explain the process of recovering layer $i$, assuming that the first $i-1$ layers of the target network $f$ have been correctly extracted. 

Signature extraction is carried out in five main steps: searching critical points, recovering partial signatures, merging partial signatures, finding missing entries in the signatures, and computing the bias.

\medskip 
\noindent \textbf{Random search for critical points.}
In each polytope, the network behaves affinely, meaning that the derivative of $f$ remains constant within the same polytope. Therefore, critical points can be identified when a change in the derivative occurs, indicating that a critical hyperplane has been crossed. To find these critical points, we apply a binary search along random lines in the input space. See Appendix~\ref{search for ccps} for a detailed explanation.

\medskip 
\noindent \textbf{Partial signature recovery.}
\label{partial_signature_recovery}
When we cross a neuron's critical hyperplane, we move, depending on the sign of the neuron, from a region (polytope) where the neuron does not contribute to the network’s output to one where it does. All other neurons remain unchanged. Consequently, by querying the network close to a critical point, we can construct a system of equations that isolates the contribution of the neuron in question, allowing us to recover its parameters up to a sign. We now describe this process in more detail.

Following~\cite{C:CarJagMir20}, we define the \emph{second-order differential operator} of $f$ along direction $\Delta$ as $$\partial^2_\Delta f(x) := \frac{1}{\epsilon_1}(f(x+ \epsilon_2 \Delta) + f(x - \epsilon_2 \Delta)- 2 f(x)).$$ Suppose the target layer is layer $i$. We assume that the previous layers have been extracted and thus that $\Gamma_x^{(i-1)}$ is known for all inputs $x$. 

\begin{lemma}[\cite{C:CarJagMir20}]
\label{lemma1}
Let $x$ be a critical point for a neuron $\eta_k^{(i)}$ in layer $i \in \{1, \ldots, r\} $, and assume that $x$ is not a critical point for the other neurons. Further, let $\Delta \in \mathbb{R}^{d_0}$. 
Then: 
\[
\partial^2_\Delta f (x) = c_k^{(i)} \big|A_k^{(i)}\cdot (\Gamma_x^{(i-1)} \Delta) \big|
\]
where $c_k^{(i)} \in \mathbb{R}$ is a constant.
\end{lemma}

\begin{proof}
See Appendix~\ref{proof_lemma}.
\end{proof}

By using Lemma \ref{lemma1} in two directions, $\Delta$ and $\Delta_0$, we obtain \[
    \frac{\partial^2_{\Delta} f (x)}{\partial^2_{\Delta_0} f (x)} = \frac{\big|A_k^{(i)}\cdot(\Gamma_x^{(i-1)} \Delta) \big|}{\big|A_k^{(i)}\cdot (\Gamma_x^{(i-1)} \Delta_0) \big|}.
\]
Then, by comparing the absolute values between 
\[\frac{\partial^2_{\Delta} f (x)+\partial^2_{\Delta_0} f(x)}{\partial^2_{\Delta_0} f(x)} \text{ and }  \frac{\partial^2_{\Delta + \Delta_0} f (x)}{\partial^2_{\Delta_0} f(x)},\]we can eliminate the absolute value (see Appendix \ref{removing_absolute_values}), obtaining \[\frac{A_k^{(i)}\cdot(\Gamma_x^{(i-1)} \Delta)}{A_k^{(i)}\cdot(\Gamma_x^{(i-1)} \Delta_0)}.\]\\ 
By taking enough directions $\Delta$ in the input space, we can solve for $A_k^{(i)}$ up to a constant $c^{(i)}_k$, reconstructing the signature (see Appendix \ref{extraction_detail}). We note $\hat A_k^{(i)}$ the recovered signature of $A_k^{(i)}$. As $k$ remains unknown, the extraction permutes the rows of $A^{(i)}$: we do not know if $\hat{A}_1^{(i)} = A_1^{(i)}$ or if $\hat{A}_1^{(i)} = A_2^{(i)}$. When extracting the next hidden layer $\hat{A}^{(i+1)}$, the end-to-end attack naturally permutes the columns to match the permuted rows of $\hat{A}^{(i)}$, since $\hat{A}^{(i)}$, rather than $A^{(i)}$, is used in the extraction. In the same fashion, the next layer absorbs the constant $c^{(i)}_k$: $\hat{a}^{(i+1)}_{j,k} = a^{(i+1)}_{j,k}\cdot\frac{c^{(i+1)}_j}{c^{(i)}_k}$. The extracted last layer forces the permutation and scaling to match the target neural network, yielding a functionally equivalent extraction.

However, since $\Gamma_x^{(i-1)} = I_x^{(i-1)} A^{(i-1)}\Gamma_x^{(i-2)}$, ReLUs on layer $i-1$ are blocking some coefficients for all directions $\Delta$ we take around $x$. Therefore, instead of extracting a full (scaled) matrix row, for example $\left(a_{1}, a_{2}, a_{3}, a_4, a_5\right)$, we can only retrieve a partial signature, say, $\left(0, a_2, a_3, 0, a_5 \right)$.
\medskip

\noindent \textbf{Merging signatures into components.}
To reconstruct the full signature, we need to merge partial signatures obtained from different critical points of the same neuron $\eta_k^{(i)}$. Each recovered signature from critical points of $\eta_k^{(i)}$ is $A_k^{(i)}$ with some missing entries, scaled to an unknown factor. Assuming that no two rows of the matrix are identical, merging two signatures requires only checking whether all their shared non-zero weights are proportional. For example, consider two signatures $(\lambda a, \lambda b, 0, \lambda d, 0)$ and $(\Lambda a, \Lambda b, 0, 0, \Lambda e)$ where $\lambda$ and $\Lambda$ are constants. If $\frac{\lambda a}{\lambda b} = \frac{\Lambda a}{\Lambda b}$, they can be merged to get $(a,b,0,d,e)$, up to a constant (see Fig.~\ref{fig:merges_visual}). The resulting merged row is called a \emph{component}. The \emph{size of a component} is the number of critical points whose partial signatures merged to form that component.

\begin{wrapfigure}{r}{6.0cm}
\centering
\vspace{-0.8cm}
\begin{tikzpicture}[scale=0.9]
    \draw (0,0) rectangle (6,3);
    \draw[thick, gray]  (0.9,3) -- (1.2,1.5) -- (0.7,0.5) -- (0,0.2);
    \node at (1.2,2.75) {\textcolor{gray}{$\eta_1$}};

    \draw[thick, gray]  (5,3) -- (4,2) -- (3.2,1.5) -- (2.7,0);
    \node at (2.5,0.25) {\textcolor{gray}{$\eta_2$}};
    
    \draw[thick, gray] (6,0) -- (4.9, 1.4) -- (2.5, 2.4) -- (1.7, 3);
    \node at (5.5,0.25) {\textcolor{gray}{$\eta_5$}};

    \draw[thick, gray] (1, 0) -- (2.2,1) -- (4,3);
    \node at (1.7,0.25) {\textcolor{gray}{$\eta_4$}};

    \draw[thick, gray] (3.5,3) -- (6,1.7);
    \node at (5.7, 2.1) {\textcolor{gray}{$\eta_3$}};

    \draw[thick, red] (0, 2.5) -- (1.07,2.2) -- (2.25,2.6) -- (3.65, 2.6) -- (4, 2) -- (6,1.2);

    \filldraw[fill=black] (1.5,2.35) circle (3pt);
    \filldraw[fill=black] (3.85,2.3) circle (3pt);
\end{tikzpicture}
\caption{Input space of a network where neurons on the previous layer are labelled in grey on their active side. In red is the neuron we aim to find. The critical points on the left and right yield respectively $(\Lambda a, \Lambda b, 0, 0, \Lambda e)$ and $(\lambda a, \lambda b, 0, \lambda d, 0)$. We can infer $(a,b,0,d,e)$ up to a constant, even though no single polytope activates $\eta_1,\eta_2,\eta_4,\eta_5$ simultaneously.}
\vspace{-0.2cm}
\label{fig:merges_visual}
\end{wrapfigure}

If a component has a size greater than 2, the original attack from Carlini et al.~\cite{C:CarJagMir20} assigns it to the target layer, based on the belief that signature extraction would not produce consistent results on critical points from deeper layers (deeper layers are not expected to have such linear dependencies). However, we will explain later why this is not a correct perception, as critical points from deeper layers might indeed be merged. The process described above (randomly searching for critical points, extracting partial rows, and merging signatures) is repeated until the number of components of size $> 2$ is equal to $d_i$, the number of neurons at layer $i$.

\medskip
\noindent \textbf{Targeted search for critical points.} At this stage of the attack, we know that all selected components correspond to neurons on the target layer. However, components recovered from random critical points often have weights missing, even after merging numerous partial signatures. To complete the partial rows, a more targeted search for critical points is necessary. We start from one of the critical points within the component. As we cross different hyperplanes associated with neurons in the previous layer, we activate different weights of the target neuron, ultimately enabling its full reconstruction.

For example, in Figure \ref{fig:merges_visual}, we would follow the red hyperplane as it bends across the grey hyperplanes, retrieving enough  critical points to reconstruct the full signature of our neuron. In higher dimensions, we follow directions that are more likely to trigger the weights we have not found yet. The exact procedure is described in \cite{C:CarJagMir20} and further improved in \cite{Hannah}.

\medskip 
\noindent \textbf{Recovering the bias.}
Once we have recovered the signature of $\eta^{(i)}_k$, $c^{(i)}_k A^{(i)}_k$, it suffices to pick a critical point $x$ of the component to compute the corresponding scaled bias $c^{(i)}_kb^{(i)}_k$ using the equation
$(c^{(i)}_kA^{(i)}_k) \cdot F^{(i-1)}(x)+c^{(i)}_kb^{(i)}_k=0$.

\medskip 
\noindent \textbf{Last layer recovery.}
\label{last layer recovery}
The oracle returns the complete raw outputs of $f$ and the last output layer is a linear layer. Therefore, its extraction is straightforward. No additional query is needed for extraction, as we can reuse previously queried critical points to build the linear system and recover the last linear layer's weights.

\subsection{Overview of Existing Sign-Extraction Methods}
\label{sign explanation}
After recovering each neuron's signature in layer up to a constant $c^{(k)}_i$, we must determine the sign of each constant. Canales-Martínez et al.~\cite{EC:CCHRSS24} propose two complementary sign-extraction techniques: the \emph{system-of-equations} (SOE) approach and the \emph{neuron wiggle}. 
We briefly summarize both methods. 
\medskip

\noindent \textbf{SOE}. The system-of-equations method aims at contractive layers, where it recovers the signs of every neuron in the target layer together.
Locally around an input $x$, for a small perturbation
direction $\Delta_k$, we have
\begin{equation*}
\label{eq:soe_local}
f(x+\Delta_k) - f(x) \;=\; G^{(i+1)}_x \, I^{(i)}_x \, A^{(i)} \, F^{(i-1)}_x \, \Delta_k,
\end{equation*}
where $G^{(i+1)}_x$ is the linear contribution of deeper layers at $x$,
$A^{(i)}$ is the weight matrix of the target layer, $I^{(i)}_x$ is the diagonal
activation mask of layer $i$ at $x$, and $F^{(i-1)}_x$ is the local extracted linear map
up to layer $i-1$.
To find the signs of all neurons, we want to find $I^{(i)}_x$. Equivalently, we want to find which entries of $G^{(i+1)}_xI^{(i)}_x$ are zero. Therefore we build the following system:
\[
\{c\cdot y_k = z_k\}_k
\]
where $c = G^{(i+1)}_xI_x^{(i)}$ is the variable we are looking for, $y_k = A^{(i)}F^{(i-1)}_x\Delta_k$ and $z_k= f(x+\Delta_k) - f(x)$. To solve SOE, we thus need that $\text{rank}(F^{(i-1)}_x)\geq \text{rank}(A^{(i)})$. This technique can therefore be used as long as the network stays very contractive. It is thus usually applicable only for the first layer.\\

\noindent \textbf{Neuron Wiggle}. For non-contractive layers, the neuron wiggle can be used to recover neuron signs individually. Consider a change in input $\Delta_k$ sufficiently small to stay within the linear region but which maximises $\hat{\eta_k}(x^*+\Delta_k)$. The general intuition is that the ReLU blocks the change caused by $\Delta_k$ on the inactive side of $\eta$. Therefore if our sign guess is correct, the output should display a large change. If our sign guess is incorrect, we shouldn't see a particular change. We can observe a correct change only if we can sufficiently maximise $\hat{\eta_k}(x^*+\Delta_k)$ 
without maximising at the same time other neurons. We refer the reader to the original article~\cite{EC:CCHRSS24} for further details. This makes the neuron wiggle a probabilistic method which depends on many factors, notably the architecture of the network. As a consequence, the neuron wiggle is ran on many critical points $x^*$ of $\eta_k$ and the sign is chosen according to the overall results, with a certain confidence level $\alpha$. $\alpha=1$ indicates that all points agree, whereas $\alpha=0$ implies an even split. Canales-Martínez et al. report very convincing results on a $3072-256^{(8)}-10$ network, recovering the sign of all but ten neurons in this very large network. The neuron wiggle is further scrutinised in~\cite{Hannah}, where the authors recommend running an exhaustive search on the signs of the neurons with confidence level $\alpha <0.75$.

%% file: content/strategies.tex
\begin{figure}[htb!]
\centering
\scalebox{0.8}{
\begin{tikzpicture}[node distance=0.70cm, mybox/.style={draw=black, thick, rounded corners, inner sep=0.5cm}]
\node (search2) [rectangle, draw] at (5,0) {random search};
\node[draw=RawSienna, very thick] (noise1) [rectangle, draw, below of=search2] {discard deeper points};
\node (signature2) [rectangle, draw, below of=noise1] {partial signature recovery};
\node[draw=pastelred, very thick] (merge2) [rectangle, draw, below of=signature2] {signature intersections};

\node (merge1) [rectangle, draw, below of=merge2] { signature merges};
\node[draw=RawSienna, very thick]  (signal2) [rectangle, draw, below of=merge1] {discard deeper components};
\node (targeted2) [rectangle, draw, below of=signal2] {targeted search};
\node[draw=black] (precision2) [rectangle, draw, below of=targeted2] {precision improvement};
\node[draw=none] (dummy2) at ($(signature2)+(1.6,0)$) {};

\draw[->] (search2) -- (noise1);
\draw[->] (noise1) -- (signature2);
\draw[->] (signature2) -- (merge2);
\draw[->] (merge2) -- (merge1);
\draw[->] (merge1) -- (signal2);
\draw[->] (signal2) -- (targeted2);
\draw[->] (targeted2) -- (precision2);

\node[draw=none] (topa1) at ($(search2)+(1.425,0.35)$) {};
\node[draw=none] (bottoma1) at ($(search2)+(1.425,-0.35)$) {};

\node[draw=none] (topa2) at ($(signature2)+(2.2,0.35)$) {};
\node[draw=none] (bottoma2) at ($(signature2)+(2.2,-0.35)$) {};
\node (textta2) [draw=none] at ($(topa2)+(1.75,-0.15)$) {\color{black}improved precision};
\node (textba2) [draw=none] at ($(bottoma2)+(1.75,0.15)$) {\color{black}(adjust scaling factors)};
\draw[->] (topa2) -- (bottoma2);

\node[draw=none] (topa3) at ($(precision2)+(2,0.35)$) {};
\node[draw=none] (bottoma3) at ($(precision2)+(2,-0.35)$) {};
\node (textta3) [draw=none] at ($(topa3)+(1.6,-0.15)$) {\color{black}improved precision};
\node (textba3) [draw=none] at ($(bottoma3)+(1.6,0.15)$) {\color{black}(avoid flat regions)};
\draw[->] (topa3) -- (bottoma3);

\node(search3) [rectangle, draw] at (0,0) {random search};
\node(empty1) [below of=search3]{};
\node (signature3) [rectangle, draw, below of=empty1, fill = pastelred] {\color{white}partial signature recovery};
\node(empty2) [below of=signature3]{};
\node (merge3) [rectangle, draw, below of=empty2] {signature merges};
\node (condition3) [rectangle, draw, below of=merge3, fill = RawSienna] {\color{white}$d_i$ $>2$ - sized components};
\node (targeted3) [rectangle, draw, below of=condition3] {targeted search};
\node (precision3) [rectangle, draw, below of=targeted3] {precision improvement};
\node[draw=none] (dummy3) at ($(signature3)+(-2.1,0)$) {};
\draw[->] (search3) -- (signature3);
\draw[->] (signature3) -- (merge3);
\draw[->] (merge3) -- (condition3);
\draw [->] (condition3) to [out=180,in=180] (search3);
\draw [->] (condition3) to (targeted3);
\draw [->] (targeted3) to (precision3);

\node[mybox, fit=(dummy2)(dummy3)(search2)(noise1)(signature2)(merge2)(signal2)(targeted2)(precision2)(search3)(empty1)(empty2)(merge3)(condition3)(targeted3)(precision3)(textta2),
      label=above:{}, 
      inner sep=0.6cm] {};

\end{tikzpicture}
}
\caption{
    Left: Original signature extraction from \cite{C:CarJagMir20}. Right: Proposed improvements. Two error-inducing steps in the original attack are coloured on the left. Improvements match the colour of the step they address. Further precision improvements for signature extraction are marked on the right.
}
\label{fig:strategies}
\end{figure}

\section{Improving Signature Extraction} \label{sec:strategies}

To the best of our knowledge, all prior work has been able to apply signature extraction on very shallow neural networks only (3 layers maximum). 
Specifically, according to Table 1 in~\cite{C:CarJagMir20}, the deepest model they attempted to attack had an architecture of $40-20-10-10-1$, consisting of only $3$ hidden layers. On the other hand, Table 2 in~\cite{Hannah} shows that they attempted to attack a model with an architecture of $784-16^{(8)}-1$, which had $8$ hidden layers. However, despite running the attack for over $36$ hours, they failed to recover the fourth hidden layer. Although~\cite{EC:CCHRSS24} demonstrated that their sign extraction was effective across all layers of a complex $3072-256^{(8)}-10$ model, they supposed in their experiments that the signatures had been at this stage correctly extracted, without running the two processes together. 

Motivated by these limitations, we analyze the signature extraction method and identify two challenges preventing its successful application on deeper layers. We then propose solutions to overcome them. The left part of Fig.~\ref{fig:strategies}  illustrates the signature extraction workflow and highlights where these challenges arise.

\begin{figure}[htb!]
    \centering
    \begin{subfigure}{0.48\textwidth}
        \centering
        \includegraphics[width=\linewidth]{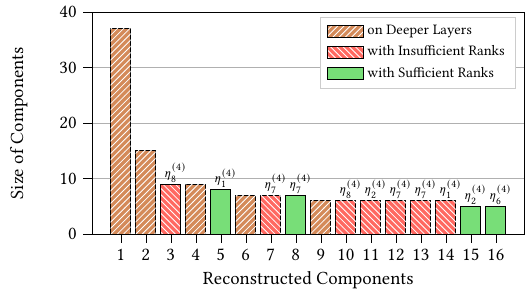}
        \caption{Original signature extraction.}
        \label{fig:benchmark}
    \end{subfigure}
    \hfill
    \begin{subfigure}{0.48\textwidth}
        \centering
        \includegraphics[width=\linewidth]{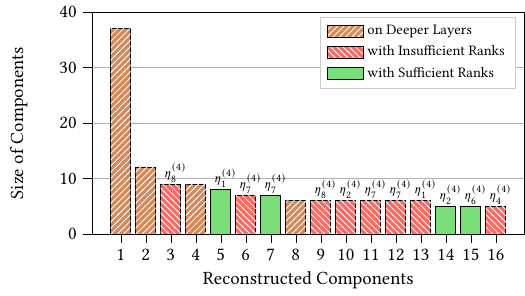}
        \caption{After discarding deeper points.}
        \label{fig:noise_filtered}
    \end{subfigure}
    \vspace{0.7cm}

    \begin{subfigure}{0.48\textwidth}
        \centering
        \includegraphics[width=\linewidth]{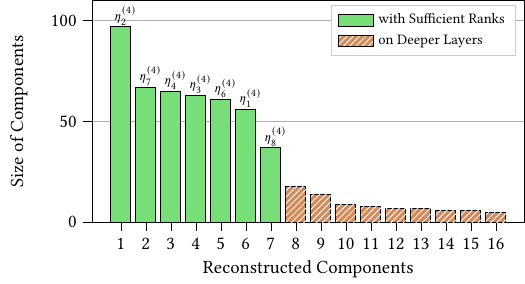}
        \caption{After signature intersections.}
        \label{fig:noise_filtered_rank_solved}
    \end{subfigure}
    \hfill
    \begin{subfigure}{0.48\textwidth}
        \centering
        \includegraphics[width=\linewidth]{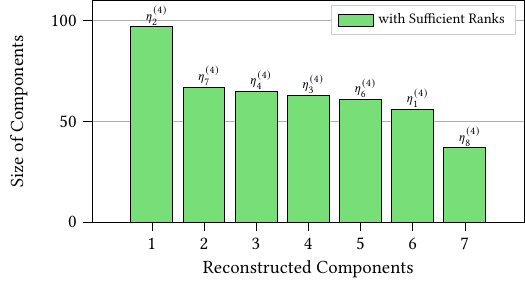}
        \caption{After discarding deeper components.}
        \label{fig:final_results}
    \end{subfigure}
    \vspace{0.3cm}
    
    \caption{Gradual improvement of signature extraction results on layer 4 of Model II with $3,000$ critical points (due to space constraints, only the largest 16 components are displayed). Each component on the target layer is labelled with its associated neuron on top. (a) Original signature extraction from \cite{C:CarJagMir20}. (b) After discarding deeper critical points (see Section~\ref{subsec:filteralgo}), deeper components either disappear or have a smaller size. (c) After intersecting critical points with insufficient ranks (see Section~\ref{rankissue}), components with rank deficiency disappear. (d) After discarding deeper components (see Section~\ref{subsec:filteralgo}), all remaining components are in the target layer. The only unrecovered component corresponds to an always-off neuron $\eta^{(4)}_5$.}
    \label{fig:deeper_layers_influence}
\end{figure}

We provide now a brief overview of each failure. First, the signature extraction relies on solving a system of equations to recover the weights of the target neuron associated with all active neurons in the previous layer. 
However, as we go deeper into the network, the rank of this system might not match the number of active neurons on the previous layer. This mismatch can lead to an underdetermined system, resulting in incorrect signature extraction. Such components are in red in Fig. \ref{fig:deeper_layers_influence}.

Second, the component selection of the largest $d_i$ components of size greater than 2 is not a good approach. Indeed, our experiments show that this strategy is only valid for shallow networks where there are comparatively fewer critical points on deeper layers. The right part of Fig.~\ref{fig:strategies} gives an outline of our improvements. Components from deeper layers are in brown in Fig. \ref{fig:deeper_layers_influence}.

Our methods for resolving the issues of incorrect signature extraction and deeper components are given respectively in Section~\ref{rankissue} and Section~\ref{subsec:filteralgo}. We show in Fig.~\ref{fig:deeper_layers_influence}, as an illustration, how each improvement we propose allows us to gradually reach a correct extraction the 4th layer of Model II with architecture $784-8^{(8)}-1$. The same approach naturally extends to subsequent layers.

Moreover, we present in Section~\ref{precision_improvements} several numerical precision‑improvement strategies that increase extraction efficiency, and provide additional precision refinements for wide layers in the context of an end-to-end attack. The rightmost part of Fig.~\ref{fig:strategies} indicates when they are applied.

\subsection{Increasing Rank with Subspace Intersections}
\label{rankissue}
\leavevmode
\begin{wrapfigure}{r}{5.5cm}
\centering
\vspace{-0.2cm}
\begin{tikzpicture}[
    scale=0.65, 
    every node/.style={scale=0.9}, 
    node distance={12mm}, 
    minimum size=0.1cm, 
    main/.style = {draw, circle}
]
\node[main, fill=gray] (1) at (0,1) {};
\node[main, fill=gray] (2) at (0,2) {};

\node[main, opacity=0.5] (3) at (2,0) {};
\node[main, fill=blue] (4) at (2,1) {};
\node[main, fill=blue] (5) at (2,2) {};
\node[main, opacity=0.5] (6) at (2,3) {};

\node[main, fill=blue] (7) at (4,0) {};
\node[main, fill=blue] (8) at (4,1) {};
\node[main, opacity=0.5] (9) at (4,2) {};
\node[main, fill=blue] (10) at (4,3) {};

\node[main, opacity=0.5] (11) at (6,0) {};
\node[main, opacity=0.5] (12) at (6,1) {};
\node[main, opacity=0.5] (13) at (6,2) {};
\node[main, fill=black] (14) at (6,3) {};

\node at (7, 0) {$\cdots$};
\node at (7, 1) {$\cdots$};
\node at (7, 2) {$\cdots$};
\node at (7, 3) {$\cdots$};
\draw[->, opacity=0.5] (1) -- (3);
\draw[->, blue] (1) -- (4);
\draw[->, blue] (1) -- (5);
\draw[->, opacity=0.5] (1) -- (6);
\draw[->, opacity=0.5] (2) -- (3);
\draw[->, blue] (2) -- (4);
\draw[->, blue] (2) -- (5);
\draw[->, opacity=0.5] (2) -- (6);

\draw[->, opacity=0.5] (3) -- (7);
\draw[->, opacity=0.5] (3) -- (8);
\draw[->, opacity=0.5] (3) -- (9);
\draw[->, opacity=0.5] (3) -- (10);
\draw[->, blue] (4) -- (7);
\draw[->, blue] (4) -- (8);
\draw[->, opacity=0.5] (4) -- (9);
\draw[->, blue] (4) -- (10);
\draw[->, blue] (5) -- (7);
\draw[->, blue] (5) -- (8);
\draw[->, opacity=0.5] (5) -- (9);
\draw[->, blue] (5) -- (10);
\draw[->, opacity=0.5] (6) -- (7);
\draw[->, opacity=0.5] (6) -- (8);
\draw[->, opacity=0.5] (6) -- (9);
\draw[->, opacity=0.5] (6) -- (10);

\draw[->, opacity=0.5] (7) -- (11);
\draw[->, opacity=0.5] (7) -- (12);
\draw[->, opacity=0.5] (7) -- (13);
\draw[->, blue] (7) -- (14);
\draw[->, opacity=0.5] (8) -- (11);
\draw[->, opacity=0.5] (8) -- (12);
\draw[->, opacity=0.5] (8) -- (13);
\draw[->, blue] (8) -- (14);
\draw[->, opacity=0.5] (9) -- (11);
\draw[->, opacity=0.5] (9) -- (12);
\draw[->, opacity=0.5] (9) -- (13);
\draw[->, opacity=0.5] (9) -- (14);
\draw[->, opacity=0.5] (10) -- (11);
\draw[->, opacity=0.5] (10) -- (12);
\draw[->, opacity=0.5] (10) -- (13);
\draw[->, blue] (10) -- (14);
\node[above of=1, node distance=1.35cm] (12) {\scriptsize{Input $x$}};
\end{tikzpicture}
\caption{Solving for a neuron’s weights from certain critical points can yield an underdetermined system.
The target neuron is shown in black; active neurons are in blue. At input $x$, $\operatorname{rank}(\Gamma^{(1)}_x)=2$, so $\operatorname{rank}(\Gamma^{(2)}_x)\le 2$ even though three neurons are active in layer $2$; consequently, the layer-2 system has no unique solution.}
\label{fig:rank_network}
\end{wrapfigure} 
\noindent \textbf{The issue: insufficient rank.} When extracting the signature, our goal is to recover a partial signature $y_x \in \mathbb{R}^{d_{i-1}}$ of a neuron $\eta^{(i)}$ at its critical point $x$.
Consequently, the solution space should be of dimension 1. To determine $y_x$, we solve $(\Gamma_x^{(i-1)} \Delta)\cdot y_x = \partial^2_{\Delta} f(x)$ using a sufficient number of random directions $\Delta \in \mathbb{R}^{d_0}$.
The resulting partial row $y_x$ consists of the coefficients for the neurons in layer $i-1$ that are active at input $x$, entries on inactive indices are zero. We denote by $s_x^{(i-1)}$ the number of active neurons in layer $i-1$ and by 
$r^{(i-1)}_{x}$ 
the rank of the map $\Gamma^{(i-1)}_{x}\in\mathbb{R}^{d_{i-1}\times d_0}$. 
If $r_x^{(i-1)} < s_x^{(i-1)}$, the system is underdetermined and $y_x$
cannot be uniquely determined.
This typically occurs when some earlier layer
$k<i-1$ has fewer active neurons than layer $i-1$ (see
Fig.~\ref{fig:rank_network}), since
$
    \Gamma_x^{(i-1)} = I_x^{(i-1)} A^{(i-1)} \cdots I_x^{(1)} A^{(1)}   
$
and thus:
\[
\text{rank}(\Gamma_x^{(i-1)}) \leq \min(\{\text{rank}(I_x^{(k)})\}_{1\leq k \leq i-1}) \neq \text{rank}(I_x^{(i-1)}) \text{ in general.}
\]

If $X$ denotes the random variable representing the number of active neurons per layer, then $\mathbb{P}(r_x^{(i-1)}< s_x^{(i-1)}) = 1 - \mathbb{P}(X\geq s_x^{(i-1)})^{i-1}$.
Because of the exponent $(i-1)$, this (failure) probability approaches $1$
rapidly with depth, making extraction beyond the first few layers increasingly unlikely.

\medskip
\noindent \textbf{Signature intersections.} If the rank is insufficient, we no longer have a unique solution but rather a space of solutions. 
For a critical point $x$ in the layer $i$, the solution space is given by $
\mathcal{S}_x = L_x + \ker(\Gamma_x^{(i-1)})$, where $ L_x$ is any particular solution and $\ker(\cdot)$ denotes the kernel.  Suppose we have another critical point $x'$ with solution space $
\mathcal{S}_{x'} = L_{x'} + \ker(\Gamma_{x'}^{(i-1)}).$ We want to intersect $\mathcal{S}_x$ and $\mathcal{S}_{x'}$ into a space of smaller dimension containing $A^{(i)}_k$. Doing this for multiple critical points should give us a space of dimension $1$, yielding $A^{(i)}_k$ up to a scaling factor.

Let's study this in more details. Since we recover rows up to a scalar factor, we seek to compute the intersection of $\mathcal{S}_x$ and $\lambda \mathcal{S}_{x'}$, where $\lambda$ is a scalar factor equal to the ratio of the row scalar factors of $x$ and $x'$. Let $(e_1, \ldots, e_k)$ be a basis of $\ker(\Gamma_x^{(i-1)})$, and let $(f_1, \ldots, f_m)$ be a basis of $\ker(\Gamma_{x'}^{(i-1)})$. We then solve the following system in $\mathbb{R}^{d_{i-1}}$:
\[ \label{eq:solutionspace}
L_x + \mu_1 e_1 + \dots + \mu_k e_k = \lambda L_{x'} + \lambda_1 f_1 + \dots + \lambda_m f_m,    
\]
where the unknowns are $\mu_1, \dots, \mu_k$ and $ \lambda, \lambda_1, \dots, \lambda_m \in \mathbb{R}$. This system consists of $d_{i-1}$ equations in $\mathbb{R}$. 
However, only the equations corresponding to the active neurons in the layer $i-1$ for both inputs $x$ and $x'$ are relevant. To ensure that we do not merge spaces from critical points associated with different neurons, we overdetermine the system. Indeed, an overdetermined system with random coefficients is inconsistent with very high probability, and thus a system of $\ell>N$ equations with $N$ unknowns typically doesn't have a solution except if there is some ground truth behind it. Therefore, we impose the following merging condition: 
\[
    \ell > 1 + |\ker(\Gamma_x^{(i-1)})| + |\ker(\Gamma_{x'}^{(i-1)})|=1+k+m, 
\]
where $\ell$ is the number of relevant equations. \\
\indent Intersections allow us to exploit at least 90\% of critical points on the target layer when attacking for example the Model II. By comparison,
the original attack gives a uniquely solvable system (and thus a correct signature) for only about \(6\%\) of layer-4 critical points, with even worse results in deeper layers.  We could intersect subspaces three by three to have a lower condition for merging, allowing for the use of more critical points. However, this would considerably slow down the attack. For this reason, all our results use pairwise merges.

We provide in Fig.~\ref{fig:noise_filtered_rank_solved} an experiment showing how intersecting the critical-point solution spaces prunes components affected by rank deficiency, making the correct ones stand out. 

\subsection{Deeper Layers' Influence on Extraction}
\label{subsec:filteralgo}

While going through the signature extraction procedure, we made the crucial assumption that $x$ is a critical point of a neuron on the layer we are targeting. Since we are trying to extract that layer, we cannot verify this assumption. 
We might be extracting a neuron on a deeper layer. The authors of \cite{C:CarJagMir20} claim that it is exceedingly unlikely that signatures extracted from critical points on a deeper layer can be merged into a component. They conclude that a component of size greater than two is on the target layer. In this section, we first explain why signatures from deeper layers can indeed be merged. Second, we propose a two-step solution which consists in discarding deeper points, then deeper components. Finally we discuss how noise behaves in very large networks.

\medskip
\noindent \textbf{The issue: deeper merges.} Assume that we are trying to extract layer $i$. From two critical points $x_1$ and $x_2$ of the same neuron in a deeper layer $i+t$, we solve respectively for partial weights $y_1$ and $y_2$ using the preceding layer $i+t-1$ in the systems below. Since $x_1$ and $x_2$ are critical points of the same neuron, and we are extracting from the preceding layer, we know that $y_1$ and $y_2$ can be merged. We'll see that under a specific scenario, the extraction from $x_1$ and $x_2$ can also be merged when using layer $i-1$ for extraction, inducing noise in our extraction. 
\begin{align*}
\partial^2_\Delta f (x_1) &= (\Gamma_{x_1}^{(i+t-1)} \Delta) \cdot y_1\\
\partial^2_\Delta f (x_2) &= (\Gamma_{x_2}^{(i+t-1)} \Delta) \cdot y_2
\end{align*}
Let's write $A_{x_1}$ the matrix $I^{(i+t-1)}_{x_1} \circ A^{(i+t-1)}\circ I^{(i+t-2)}_{x_1}\circ \dots \circ I^{(i)}_{x_1}\circ A^{(i)}$ and $A_{x_2}$ the matrix $I^{(i+t-1)}_{x_2} \circ A^{(i+t-1)}\circ I^{(i+t-2)}_{x_2}\circ \dots \circ I^{(i)}_{x_2}\circ A^{(i)}$.
Thus, using $A_{x_1}$ and $A_{x_2}$, we can rewrite
\begin{align*}
\partial^2_\Delta f (x_1) &= [(A_{x_1}\circ\Gamma_{x_1}^{(i-1)}) \Delta] \cdot y_1\\
\partial^2_\Delta f (x_2) &= [(A_{x_2}\circ\Gamma_{x_2}^{(i-1)}) \Delta] \cdot y_2
\end{align*}
Now suppose that $x_1$ and $x_2$ have the same activation pattern between layers $i$ and $i+t-1$, meaning they set the same neurons as active. In this case, $A_{x_2} = A_{x_1}$, which we write as $A$. Therefore, our systems are as follows.
\begin{align*}
\partial^2_\Delta f (x_1) &=  [(A\circ\Gamma_{x_1}^{(i-1)}) \Delta] \cdot y_1\\
\partial^2_\Delta f (x_2) &=  [(A\circ\Gamma_{x_2}^{(i-1)}) \Delta] \cdot y_2
\end{align*}
These systems respectively yield the same solutions as:
\begin{align*}
\partial^2_\Delta f (x_1) &= (\Gamma_{x_1}^{(i-1)} \Delta) \cdot (A^{\top}y_1)\\
\partial^2_\Delta f (x_2) &= (\Gamma_{x_2}^{(i-1)} \Delta) \cdot (A^{\top}y_2).
\end{align*}
Since $y_1$ and $y_2$ can merge, so can $A^{\top}y_1$ and $A^{\top}y_2$. These systems correspond exactly to the partial signature extraction from $x_1$ and $x_2$ when extracting layer $i$. This is why the partial signatures extracted from $x_1$ and $x_2$ can merge into a component even though they are not on the target layer. We call all these unwanted additional components from deeper layers \emph{noise components}.

\begin{wrapfigure}{R}{5.5cm}
\centering
\scalebox{0.85}{
\begin{tikzpicture}[scale=1.0]
    \draw (0,0) rectangle (6,3);
    
    \draw[black]  (0.8,3) -- (1,1.5) -- (0.5,0.6) -- (1,0);
    \node[right] at (0.8,2.7) {\textcolor{black}{$\eta_1^{(i-1)}$}};

    \draw[black]  (4.8,3) -- (5.8,1.5) -- (5.5,0.6) -- (4,0);
    \node[right] at (5,2.7) {\textcolor{black}{$\eta_2^{(i-1)}$}};

    \draw[blue]  (2.7,3) -- (3.9,2.2) -- (4.75, 0.3) --(6,0.1);
    \node[right, blue] at (3.3, 2.7) {\textcolor{blue}{$\eta^{(i+j)}$}};
    \draw[red]  (0,1.7) -- (0.75,1.05) -- (4.4,1.05) -- (5.47, 2) --(6,2.2);
    \draw[dashed, thick, red] (0.75,1.05) -- (5.65,1.05);
    \node[] at (3,1.3) {\textcolor{red}{$\eta^{(i+k)}$}};
    \filldraw[fill=black] (1.75,1.05) circle (3pt);
    \node[] at (1.75,0.80) {\textcolor{black}{$x$}};
    \filldraw[fill=black] (5.65,1.05) circle (1.5pt);
    \node[] at (5.8,0.80) {\textcolor{black}{$a_2$}};
    \filldraw[fill=black] (0.75,1.05) circle (1.5pt);
    \node[] at (0.9,0.80) {\textcolor{black}{$a_1$}};
\end{tikzpicture}}
\caption{\label{fig:iff_noise_test}
Identifying if $x$ is on the target layer. $0\leq j<k$. By finding that the intersection point $a_2$ is not on the extracted hyperplane, we infer that the hyperplane we extracted from $x$ ({\color{red}- - -}) broke on a layer $i+j$ we did not extract ({\color{blue}---}). Thus, $x$ cannot be on layer $i$. 
Yet, finding that $a_1$ is on the extracted hyperplane does not give any information about $x$'s layer.
}
\end{wrapfigure}

\medskip
\noindent \textbf{Discarding deeper points.} To discard points on deeper layers, we recycle and improve an algorithm from Section 4.4.2 of~\cite{C:CarJagMir20} that was only used in the context of the targeted search. For each critical point found, the algorithm process involves: 1) computing a large number of distinct intersection points (we use 100 in practice) between the hyperplane extracted from the critical point under evaluation and the extracted network, and 2) verifying on the target neural network whether these intersection points still lie on the extracted hyperplane. One intersection point not belonging to the extracted hyperplane indicates that the hyperplane has bent on a neuron that has not yet been extracted, and hence that the extracted hyperplane is on a deeper layer. This test is not an if-and-only-if condition, as a hyperplane not breaking does not guarantee that it corresponds to a neuron on the target layer, see Fig.~\ref{fig:iff_noise_test}. This method reduces the noise, but it is far from sufficient, see Fig.~\ref{fig:noise_filtered}. We refer the reader to the original description for a more thorough account of the strategy.

\medskip
\noindent \textbf{Discarding deeper components.} \\ After discarding points on deeper layers, we compute intersections as described in Section~\ref{rankissue}, see Fig.~\ref{fig:strategies}. Intersections deal with components having insufficient rank. However, some components on deeper layers remain, see Fig.~\ref{fig:noise_filtered_rank_solved}. We use three criteria to discard them. First, we know that critical points that merge usually belong to the same neuron, regardless of whether they are in the target layer or a deeper layer. Therefore, the more critical points identified from deeper layers by the above test that merge with our candidate component, the more likely it is that this component belongs to a deeper layer. For this reason, we compute for each component a noise ratio:
$$\tau = \frac{\text{\#merges with deeper critical points}}{\text{size of the component}}.$$
Second, because merges of deeper critical points come from critical points with the same partial activation pattern, deeper components have a harder time obtaining a diverse set of critical points to yield a neuron with a very small number of entries unrecovered. Finally, because of the constraint on the activation pattern of deeper merges, deeper components tend to be smaller in size. While none of these criteria is an if-and-only-if test, combining them yields satisfying results.

Therefore, we first remove components with a size smaller than a fraction of the largest component (we found that in practice, $0.1$ was sufficient for networks with eight hidden layers). Then, we discard components that either have a $\tau$ above a certain threshold (we use $0.1$ and $0.2$ for networks with width $8$ and $16$, respectively) or a large number of entries unrecovered (at least half of them) if $\tau>0$. Notably, making a strict association between deeper critical points and incorrect components can lead to mislabelling (example in Appendix~\ref{proof_noise2}).

\medskip

\noindent\textbf{Noise in large networks.} We have shown above that components coming from deeper layers result from critical points with the same activation pattern between the target layer and their layer. As the width of the network increases, the neurons on each layer cut the input space into smaller polytopes, making it more unlikely that critical points share the same partial activation pattern, up to the point of noise not being an issue anymore for a large enough width. For example, when extracting layer $4$ of a $784-256^{(16)}-1$ network with $100,000$ critical points, we get 189 components of size larger than~$2$. 
Only three of these do not belong to the target layer, with sizes 
$3$, $3$, and $4$. By contrast, $177$ components have 
size at least $5$, and $157$ have size at least $20$.
This shows that, in wide networks, imposing a minimum component size 
is an effective way to filter out noise.

As network width grows, intersection computations become a problem. With width $256$, each critical point gives many recovered weights. This allows previously unavailable statistical arguments by viewing recovered weights as realisations of random variables. Indeed, the contribution of a neuron on a deeper layer corresponds to a very different mathematical expression from a neuron on the target layer. The deeper the neuron we are actually extracting upon, the higher the expected variance of the recovered weights is. Therefore, recovered neurons with unusual high variance can immediately be deemed to be on a deeper layer, without requiring any queries or computation, thereby increasing the proportion of critical points on the target layer. The time and queries for finding critical points grows linearly with the number of points we require, while computing intersections is comparatively very slow and grows polynomially with the number of points processed. For this reason increasing the proportion of critical points on the target layer, even if we discard some of them, results in a considerable speed-up of the attack. See Appendix \ref{Appendix Noise for large networks} for proofs, experiments and more details.

\subsection{Numerical Precision in the Context of an End-to-End Attack}
\label{precision_improvements}
In practical end-to-end attacks, model extraction relies on finite-precision arithmetic (typically 64-bit floating-point). However, many algorithms in the extraction process are numerically unstable, introducing small errors in the recovered weights. Such errors accumulate across layers and block the extraction of deeper layers. Thus, it is crucial to refine the recovered weights to minimise imprecision. 

Carlini et al.~\cite{C:CarJagMir20} introduce a method to improve the precision of recovered signatures. It consists in finding more critical points, at least $d_{i-1}$ (the number of neurons in layer $i-1$) for each neuron $\eta_k^{(i)}$ in the target layer $i$, and then solving the following linear system for its weight vector $\hat{A}^{(i)}_k$ and bias $\hat{b}^{(i)}_k$:

\begin{equation}\tag{$\star$}\label{eq:precision_system}
\begin{pmatrix}
h_1 & 1\\
h_2 & 1\\
\multicolumn{2}{c}{\cdots} \\
h_{d_{i-1}} & 1
\end{pmatrix}
\begin{pmatrix}
\hat{A}^{(i)}_k \\
\hat{b}^{(i)}_k
\end{pmatrix}
= 0
\end{equation}
where $h_j=\hat{F}^{(i-1)}(x_j)$ and $x_j$ is a critical point for neuron $\eta_k^{(i)}$. This leads to precision improvements as weights given by components come from multiple imprecision-inducing computations (derivatives, system solving, etc.) which themselves depend on the precise identification of critical points.

To find more critical points, Carlini et al.~\cite{C:CarJagMir20} use a kind of targeted search. Foerster et al.~\cite{Hannah} note that this search often struggles in practice and propose refinements, but these add numerous parameters hindering the reproducibility and substantially increasing the runtime. They report precision-improvement runtimes of 33 times that of signature extraction and 17 times that of sign extraction. For this reason,~\cite{Hannah} recommends avoiding precision improvements, while noting that it is a ``remaining concern'' for an end-to-end attack. 

Our improved signature extraction can efficiently avoid this issue. Specifically, we explicitly generate a larger pool of critical points at the beginning, a number far exceeding that required in Carlini et al.~\cite{C:CarJagMir20}. These points primarily serve to filter deep points and components in the early extraction steps, and are sufficient to be reused for precision refinement, eliminating the need for a dedicated, parameter-heavy targeted search. Our precision improvement step therefore remains a small fraction of the total extraction.

Carlini et al.'s precision improvement~\cite{C:CarJagMir20} is very potent but can only be applied at the end of the extraction of each layer. In other words, the attack per layer is still ran on imprecise numbers. Let's see why this can be a problem, before presenting our improvement to Carlini et al.'s method for wide layers.

\medskip

\noindent \textbf{Imprecision due to scaling factors affects extraction efficiency.} When recovering a partial weight-vector $y_x$ at a critical point $x$ from a linear equation system $(\hat{\Gamma}_x^{(i-1)} \Delta)\cdot y_x = \partial^2_{\Delta} f(x)$, each output from the previous layer is scaled by a neuron-specific factor, denoted as $(c_1^{(i)},c_2^{(i)},\cdots,c_{d_i}^{(i)})$. Large disparities among these scaling factors can make the matrix $\hat{\Gamma}_x^{(i-1)} \Delta$ ill-conditioned, amplifying numerical errors so that even tiny errors in recovered weights $\hat{\Gamma}_x^{(i-1)}$ cause large deviations in the partial weight-vector $y_x$. Critical points with such imprecise $y_x$ cannot be merged into components, forcing the collection of more critical points initially to assemble a component capable of recovering all desired weights. As discussed in Section~\ref{subsec:filteralgo}, computing intersections is already slow and grows polynomially with the number of points processed; imprecise partial weights further increase the number of queries and the cost of signature intersections, reducing overall efficiency.

To address this, we normalize each column of $\hat{\Gamma}_x^{(i-1)} \Delta$ by multiplying it by a secondary scaling factor $(c_1^{(i)'},c_2^{(i)'},\cdots,c_{d_i}^{(i)'})$, bringing all elements within $[-1,1]$. After solving the normalized linear equation system, these factors are multiplied back onto the solution to obtain the recovered signature. This normalization ensures more stable and precise signature extraction, see Appendix~\ref{sclaing_factor_precision} for an example. For instance, when our normalization method is applied to layer 3 of Model II, the precision of partial weights improves by roughly 10 times. The proportion of critical points merged into components increases from $88.9\%$ to $96.2\%$, capturing nearly all points associated with the target-layer neurons.

\medskip

\noindent \textbf{Imprecision from critical points in large networks.} 
We identify critical points by computing expected intersections of lines given by output derivatives, see Appendix \ref{search for ccps}. The search for critical points can therefore be imprecise if the change in linearity caused by a critical hyperplane is small, i.e. if the output region is rather flat, see Fig.~\ref{flat_regions} for illustration. We thus classify critical points based on the angles between the second derivatives used to identify them. We then only stack up critical points with large derivative differences to construct the matrix in (\ref{eq:precision_system}). For networks with a small width of 8, the number of flat critical points originally used in (\ref{eq:precision_system}) is small and does not affect much the precision of the recovered neurons. Yet, we notice significant improvement for larger widths, for example a $\times 3$ precision improvement on the crucial first layer of our Model II.\\%
\vspace{-0.5cm}
\begin{figure}[htb!]
    \centering
    \begin{subfigure}{0.45\textwidth}
        \centering
        \begin{tikzpicture}[scale=1]
            \draw[thick] (1,1) -- (-2,-2);
            \draw[thick] (-1,1) -- (2,-2);
            \draw[thick, dotted] (0,0) -- (0,-2);
            \node[right] at (0,-2.04) {$x$};
            \draw[thick, dashed, red] (-1, 0.25) -- (2, -2);
            \draw[thick, dotted, red] (-0.2857,-0.285) -- (-0.2857,-2);
            \node[left, red] at (-0.2857,-2) {$x^*$};

        \end{tikzpicture}
    \end{subfigure}
    \hfill
    \begin{subfigure}{0.45\textwidth}
        \centering
        \begin{tikzpicture}[scale=1]
            \draw[thick] (1,0.4) -- (-2,-0.8);
            \draw[thick] (-1,0.4) -- (2,-0.8);
            \draw[thick, red, dashed] (-1,-0.05)-- (2,-0.8);
            \draw[thick, dotted] (0,0) -- (0,-2);
            \draw[thick, dotted, red] (-0.461,-0.184) -- (-0.461,-2);
            \node[left, red] at (-0.461,-2) {$x^*$};
            \node[right] at (0,-2.05) {$x$};
        \end{tikzpicture}
    \end{subfigure}

    \caption{Illustration of imprecisions in flat regions when finding critical points. In both cases the attacker computes a slightly imprecise derivative (dashed red line). Though the angular error (angle between the dashed red line and the black line next to it) is the same in the two cases, the distance between the extracted critical point $x^*$ and the real one $x$ changes depending on how flat the region is. The worst errors are caused by flat regions. For this reason we avoid them.}
    \label{flat_regions}
\end{figure}
\vspace{-0.5cm}

\section{Confident Sign Extraction} \label{sec:signs}
In their original presentation of the neuron wiggle, the authors of~\cite{EC:CCHRSS24} managed to recover almost all neuron signs for a $3072-256^{(8)}-10$ neural network, assuming all signatures had been correctly recovered. It failed for neurons with so-called low confidence. This is why the authors of~\cite{EC:CCHRSS24} recommend using more resources to ensure the sign extraction of the $10\%$ lowest-confidence neurons. However, the authors of~\cite{Hannah} found that in practice, on smaller architectures, many neurons have low confidence and, worse, increasing
resources does not improve it. They resort to an exhaustive search over the signs of the low-confidence neurons to ensure successful recovery. Indeed, making a mistake on the sign of a neuron typically leads to a failure of the signature recovery of the following layer. This calls for a polynomial method to find the sign of low-confidence neurons.

We propose to combine SOE and the neuron wiggle to obtain a sign-extraction strategy which uses neurons with higher confidence than the one used by the neuron wiggle alone. 
We set-up the SOE under-determined system at a point $x$, and then use the result of the neuron wiggle on high-confidence neurons to \emph{complete} the rank: we choose the neuron inactive at $x$ with the highest confidence, we
eliminate its variable from the SOE unknowns. We iteratively add the most confident inactive neurons until the system reaches sufficient rank. Solving this system then gives the signs of the remaining low-confidence inactive neurons at $x$ alongside all the active neurons, overcoming the limitations of the neuron wiggle. For this strategy to succeed, we need a point $x$ where the number of neurons on the target layer inactive at $x$ is at least $d_i$ minus the SOE rank. We found such points easily for all networks considered by sending random input points through the extracted network. We also give
a way to increase the rank of SOE. This rank-increasing method pushes further the value of combining SOE and the neuron wiggle, increasing the confidence of sign recovery. \\
As explained in Section \ref{sign explanation}, SOE consists in solving at a point $x$ and for multiple directions $\Delta_k$ the system,
\[
\{G_{x}^{(i+1)}I_{x}A^{(i)}F^{(i-1)}\Delta_k = f(x+\Delta_k)-f(x)\}_k,
\]
where the unknown is $G_{x}^{(i+1)}I_{x}$. The rank of the system is limited by the rank of $F^{(i-1)}$ around $x$. Crucially, because the unknown is $G_{x}^{(i+1)}I_{x}$, any critical point with the same later activations (i.e., the same $G^{(i+1)} I^{(i)}$) can be used even if the previous activation patterns differ. This implies that we do not have to build the system around a single critical point. Instead, we search for different linear regions nearby which keep later activations fixed. This triggers many different activation patterns for $F^{(i-1)}$, their union having a higher rank, see Fig.~\ref{fig:soe_points}. More details are given in Appendix~\ref{SOE rank}. We call this strategy of building the SOE system not from a single critical point but from a large zone an extension of SOE.\\
\begin{figure}[htb!]
\centering
\begin{tikzpicture}[scale=0.8]

    \draw[red]  (0,2) -- (1,3) -- (3,3) -- (4,2.5) -- (5,1.5) -- (4,0) -- (1,0) -- (0.2,1) -- (0,2);
    \node[right, red] at (1.6,3.3) {\textcolor{red}{$\eta^{(i+j)}$}};

    \draw[blue] (-1,2.1) -- (6,1.4);
    \node[right, blue] at (-1,2.5) {\textcolor{blue}{$\eta_1^{(1)}$}};
    \draw[blue] (0.41,-0.49) -- (4.54,2.95);
    \node[right, blue] at (-0.2,-0.3) {\textcolor{blue}{$\eta_2^{(1)}$}};

    \draw[blue] (1,3.6) -- (1,1.9) -- (-0.175,0.57);
    \node[right, blue] at (1,3.5) {\textcolor{blue}{$\eta_1^{(2)}$}};
    \draw[blue] (2.76,3.6) -- (3,3) -- (3.4,2) -- (4.224,1.57) -- (3.87,-0.82);
    \node[right, blue] at (3.9,-0.2) {\textcolor{blue}{$\eta_2^{(2)}$}};

    \draw[black] (-0.4,0) -- (5,2.4);
    \filldraw[fill=red] (0.63,0.46) circle (1.5pt);
    \filldraw[fill=black] (2.598,1.33) circle (1.5pt);
    \filldraw[fill=black] (3.35,1.66) circle (1.5pt);
    \filldraw[fill=black] (3.72,1.83) circle (1.5pt);
    \filldraw[fill=red] (4.38,2.12) circle (1.5pt);

    \draw[black, dashed] (0.96,-0.51) -- (3.8,3.2);
    \filldraw[fill=red] (1.357,0) circle (1.5pt);
    \filldraw[fill=black] (1.97,0.81) circle (1.5pt);
    \filldraw[fill=black] (2.67,1.73) circle (1.5pt);
    \filldraw[fill=black] (3.22,2.45) circle (1.5pt);
    \filldraw[fill=red] (3.458,2.77) circle (1.5pt);

    \draw[black, dashed] (0.5,3) -- (3.857,-0.494);
    \filldraw[fill=red] (0.745,2.745) circle (1.5pt);
    \filldraw[fill=black] (0.995,2.484) circle (1.5pt);
    \filldraw[fill=black] (1.616,1.838) circle (1.5pt);
    \filldraw[fill=black] (2.323,1.1) circle (1.5pt);
    \filldraw[fill=red] (3.38,0) circle (1.5pt);

    \filldraw[fill=green] (2.25,1.18) circle (2pt);
\end{tikzpicture}
\caption{\label{fig:soe_points}
Identifying critical points that share the same activation pattern after the target layer $i$. First, we randomly select a line ({\color{black}---}) in the input space and identify the largest continuous segment of critical points (\textbullet) from previous layers that lies between two critical points (\textcolor{red}{\textbullet}) in deeper layers. Although these points differ in activation patterns in the previous layers, they share the same activation pattern in the deeper layers. We then take the midpoint (\textcolor{green}{\textbullet}) of this segment as the center of a sphere and explore multiple random directions ({\color{black}- - -}) to find more critical points from previous layers within this region.
}
\end{figure}

\indent This yields two different strategies depending whether we favour a large zone to extend to, or an initial point with a high SOE rank and a large number of confident inactive neurons. Both strategies start by computing the neuron wiggle for all neurons. If we favour the former, the extension strategy, we first run random lines through the input space, then we perform the extension on the most promising zones (see Fig. \ref{fig:soe_points}). We then complete with the highest confidence neurons and select the zone with the highest lowest confidence used. The earlier we are in the extraction, the smaller the equivalence classes defined by $x_1\sim x_2 \iff \{G^{(i+1)}_{x_1}I_{x_1} = G^{(i+1)}_{x_2}I_{x_2}\}$ are, and the less effective this increase in the rank of SOE is. This strategy works at its best in the deeper layers. If we favour the latter, the confidence strategy, we run random points through the network, compute the minimum number of active neurons per layer (the rank of SOE) and we select the point which maximises the lowest confidence used to solve the signs. We then perform an extension on that point before solving for the sign. This strategy works at its best in the first layers where the rank of SOE is still high. The authors of \cite{EC:CCHRSS24} argue that the rank of SOE stabilises in the deeper layers. Surprisingly, both strategies gave similar results in our experiments though they rely on two different structural properties of the target network. The extension strategy relies on the number of previous layer critical points within possible extensions whereas the confidence strategy relies on the rank of SOE as we progress through the network. An attacker can run both strategies and pick for each layer the one giving the more confident results. Variants between the two strategies can also be done (extension on more points than the $x$ given by the confident strategy, assessing the SOE rank of the midpoint of the extension strategy etc.). We refer to all these new sign recovery strategies as SOE+Wiggle. For the sake of consistency, unless stated, all our experiments use the extension strategy. This is because for networks our end-to-end attack targets, the extension strategy can completely avoid the neuron wiggle for many layers, giving a deterministic method for sign recovery. 

\indent We compare our method with prior work on Model II in Table~\ref{fig:sign_recovery_comparison}. Using the neuron wiggle alone,
signs can be recovered only in layers 6 and 7; in these layers, the
confidence falls below the 0.75 minimum threshold recommended by~\cite{Hannah}. In contrast, SOE+Wiggle recovers all signs correctly across the network with high confidence and without resorting to exhaustive search. For large-scale networks, SOE+Wiggle remains effective. In Table~\ref{fig:sign_recovery_comparison_large}, we compare our method with prior work on the $3072-256^{(8)}-10$ model used in Table 4 of~\cite{EC:CCHRSS24}. When computing each neuron’s confidence using the neuron wiggle, we limit the total budget to $10^6$ critical points, instead of collecting 200 points per neuron as in~\cite{EC:CCHRSS24}. We set this limited budget because some neurons would require an excessive number of queries to obtain 200 critical points. With this budget, we still ensure that every neuron has a number of critical points that exceeds the threshold recommended in~\cite{Hannah} for stable confidence estimates. Using the neuron wiggle alone, some signs in layers 5, 7, and 8 are incorrectly recovered, and the lowest confidence in each layer is low (around its minimum value 0.5). In contrast, our SOE+Wiggle method recovers all signs correctly across the network with high confidence.

\begin{table}[htb!]
\caption{
    Comparison of sign-extraction methods on Model II. When recovering the signs of neurons on the target layer, we assume that all preceding weights have been correctly extracted. 
}
\vspace{0.2cm}
\resizebox{\hsize}{!}{  
\centering
\setlength{\tabcolsep}{6pt}
\begin{tabular}{lllllllll}
\toprule
\textbf{Method} & \textbf{Metric} & \textbf{L2$^*$} & \textbf{L3} & \textbf{L4} & \textbf{L5} & \textbf{L6} & \textbf{L7} & \textbf{L8} \\ 
\midrule
Neuron Wiggle~\cite{EC:CCHRSS24}  & correct sign extraction & $7/8$ & $5/8$ & $5/8$ & $6/8$ & $8/8$ & $8/8$ & $5/8$ \\ \hdashline 
\makecell[l]{Neuron Wiggle +\\ exhaustive search~\cite{Hannah}} & low-confidence neurons ($\alpha<0.75$) & $4/8$ & $5/8$ & $5/8$ & $4/8$ & $4/8$ & $3/8$ & $4/8$ \\ 
\midrule
SOE + Wiggle & correct sign extraction & $8/8$ & $8/8$ & $8/8$ & $8/8$ & $8/8$ & $8/8$ & $8/8$ \\
\textbf{(our method})  & lowest confidence used & $\infty^\dagger$ & $0.81$ & $\infty$ & $\infty$ & $0.75$ & $1.0$ & $\infty$ \\ \hdashline 

\bottomrule

\end{tabular}
}
\flushleft{\small  
    $^*$L stands for layer.
    $^\dagger$When the SOE matrix reaches full rank after extension, the sign extraction is deterministic and the neuron wiggle is unnecessary. 
}
\label{fig:sign_recovery_comparison}
\end{table}

\begin{table}[htb!]
\vspace{-0.2cm}
\caption{Comparison of different sign-extraction methods on the $3072-256^{(8)}-10$ model used in Table 4 of~\cite{EC:CCHRSS24}. When recovering the signs of neurons on the target layer, we assume that all preceding weights have been correctly extracted.}
\vspace{0.2cm}
\resizebox{\hsize}{!}{  
\setlength{\tabcolsep}{6pt}
\begin{tabular}{llcccccccc}
\toprule
\multirow{2}{*}{\textbf{Method}} 
& \multirow{2}{*}{\textbf{Metric}}  
& \multicolumn{8}{c}{\textbf{CIFAR10-256$^{(8)}$-10}} \\ 
\cline{3-10} 
& 
& \textbf{L2$^*$} 
& \textbf{L3} 
& \textbf{L4} 
& \textbf{L5} 
& \textbf{L6} 
& \textbf{L7} 
& \textbf{L8} 
& \textbf{Recovery failures} \\ 
\midrule
\makecell[l]{Neuron Wiggle~\cite{EC:CCHRSS24}} 
& lowest confidence used
& 0.53 & 0.53 & 0.51 & 0.50 & 0.50 & 0.51 & 0.51 & 15 \\ 
\midrule
\multirow{2}{*}{Confidence strategy} 
& lowest confidence used
& 0.87 & 0.83 & 0.75 & 0.67 & 0.64 & 0.62 & 0.78 & \multirow{2}{*}{0}  \\
& confidence increases from ext.$^\triangle$ 
& --$^\dagger$ & -- & -- & 0.05 & 0.08 & 0.07 & 0.19 & \\
\hdashline
\multirow{2}{*}{Extension strategy} 
& lowest confidence used
& 0.86 & 0.80 & 0.64 & 0.62 & 0.56 & 0.66 & 0.76 & \multirow{2}{*}{0} \\
& confidence increases from ext. 
& -- & -- & 0.09 & $\infty^\ddagger$ & $\infty$ & 0.15 & $\infty$ &  \\
\bottomrule
\end{tabular}
}
\flushleft{\small  
$^*$ L stands for layer. $^\triangle$ ext. stands for the extension of SOE. $^\dagger$ When the SOE extension does not increase the SOE rank, the lowest confidence level used remains unchanged. $^\ddagger$ Without SOE extension, the extension strategy fails as the SOE rank plus the number of inactive number at the point $x$ is still smaller than $d_i$. 
}
\label{fig:sign_recovery_comparison_large}
\end{table}

\section{Evaluation Metrics and Unrecovered Weights} \label{sec:eval_weights}
\subsection{Evaluation Metrics}
\label{sec:evaluation}
In practice, an attacker cannot query the entire input domain $\mathbb{R}^{d_0}$. Whether by design or due to constraints, extraction is restricted to a smaller
region $S\subseteq \mathbb{R}^{d_0}$. On this space, which we require to be connected, some neurons do not have critical points and are thus always active (on) or always inactive (off). Always-off neurons can safely be ignored as they do not contribute to the network for inputs in $S$. Always-on neurons act as a fixed linear map on $S$. Under our goal of functional equivalence on $S$, this is naturally absorbed when recovering the final layer of $f$. There is no need to identify these neurons individually. If neuron $\eta^{(i)}_k$ is always off on $S$, then the $k$-th column of $A^{(i+1)}$ never contributes to the output on $S$, and recovering those weights is unnecessary. Many situations can cause a weight to never influence the network’s output on $S$. Thus, we introduce two definitions formalizing what the attacker is, in practice, aiming to recover.

\begin{definition}[effective architecture]
Let $f$ be a ReLU network with architecture $[d_0,d_1,\dots,d_r,d_{r+1}]$
and let $S\subseteq\mathbb{R}^{d_0}$. For $i=1,\dots,r$, let $n_i$ denote
the number of neurons $\eta^{(i)}$ in layer $i$ whose pre-activation has a
\emph{non-trivial} zero set on $S$, i.e., $
\{\,x\in S \mid \eta^{(i)}\!\circ F^{(i-1)}(x)=0\,\}\notin\{\emptyset,S\}$. Equivalently, $n_i$ counts the neurons that are neither always off nor always
on over $S$. The \emph{effective architecture} of $f$ on $S$ is then defined as $[d_0,\,n_1,\,n_2,\,\dots,\,n_r,\,d_{r+1}]_S$.
\end{definition}

\begin{definition}[effective weight]
Let $A^{(i)}=(a^{(i)}_{j,k})$ be the weight matrix from layer $i-1$ to layer $i$
and let $S\subseteq\mathbb{R}^{d_0}$. We say that the weight $a^{(i)}_{j,k}$ is
\emph{effective on $S$} if it can influence the network’s output on $S$, that is, if there exists $x\in S$ such that 
\[
\big(\eta^{(i-1)}_{k}\!\circ F^{(i-2)}(x) > 0\big)
\ \text{ and } \ 
\big(\eta^{(i)}_{j}\!\circ F^{(i-1)}(x) > 0\big).
\]
We also say that $a^{(i)}_{j,k}$ is effective on $x$ as a shorthand for effective on $\{x\}\subseteq S$.
\end{definition}

Now that we know what we aim to recover, i.e., effective weights in a network which behaves according to its effective architecture, we need metrics indicating the quality of the extraction. The number of effective weights recovered is not sufficient. Indeed, not all weights contribute equally to the network’s behaviour: failing to recover 1\% of the weights does not generally mean that the extracted model will behave correctly on 99\% of $S$. While~\cite{C:CarJagMir20} introduced the notion of $(\epsilon, \delta)$-functional equivalence to go beyond weights count, we propose a complementary metric, that we call \textit{coverage}, to quantify the fraction of $S$ on which all effective weights are recovered.

\begin{wrapfigure}{r}{6cm}
    \centering
    \vspace{-0.8cm}
    \include{tikzplots/epsilon_delta}
    \vspace{-1.0cm}
    \caption{$(\epsilon, \delta)$ analysis of our end-to-end extractions of Models I and II. As the reduced-space proportion $\delta$ increases, the maximum output error $\epsilon$ decreases. The relationship is initially very sharp but smooths out after a transition phase, which aligns with the model coverage.}
    \vspace{-0.3cm}
    \label{fig:coverage_point}
\end{wrapfigure}

\begin{definition}[coverage] Let $S'\subseteq S\subseteq \mathbb{R}^{d_0}$ be the set of all $x$ for which no unrecovered weight of the model (resp. layer) is effective. The \emph{model (resp. layer) coverage} is the proportion of points of $S$ that are also in $S'$. We refer to $S'$ as the \emph{covered space}.
\end{definition}

Representing volumes, coverage is estimated by sending a large set of random input points through the network and checking how many of them rely on unrecovered weights. The main purpose of coverage is to distinguish errors due to missing information (unrecovered but effective weights) from those due to numerical imprecision, since coverage is independent of the extraction's numerical precision. 
Unlike an $(\epsilon,\delta)$ assessment (see Def.~\ref{def:fe}), which requires choosing a tolerance $\epsilon$ and whose meaning depends on the task/scale (the same $\epsilon$ can have a very different interpretation across networks), coverage is parameter-free and task-agnostic. Setting $\delta$  to $1-$~coverage and computing $\epsilon$ over the covered space provides a cleaner view of precision-improvement strategies: in practice, $\epsilon$ stabilizes quickly once $1-\delta$ falls below the coverage (Fig.~\ref{fig:coverage_point}).

\subsection{Unrecovered Effective Weights}
Our aim is to increase the coverage by recovering as many effective weights as possible given an attack cost. Let's discuss the effective weights we do not recover. They fall in two distinct categories.
\begin{definition}[taxonomy of unrecovered weights]
\label{table:taxonomy of unrecovered weights}
 For a given weight $a^{(i)}_{j,k}$, let $S_{(+,+)}\subseteq S\subseteq \mathbb{R}^{d_0}$ be the set of input points activating the two neurons connected by $a^{(i)}_{j,k}$:
 \[
 S_{(+,+)} = \{x\in S|(\eta_k^{(i-1)}\circ F^{(i-2)}(x)>0) \land (\eta_j^{(i)}\cdot F^{(i-1)}(x)>0)\}
 \]
 Similarly, let $S_{(+,-)}\subseteq S$ be the set of input points activating, out of the two neurons connected by $a^{(i)}_{j,k}$, only the neuron on layer $i-1$:
 \[
 S_{(+,-)} = \{x\in S|(\eta_k^{(i-1)}\circ F^{(i-2)}(x)>0) \land (\eta_j^{(i)}\cdot F^{(i-1)}(x)\leq 0)\}
 \]
We define the $(+,+)$ (resp. $(+,-)$) activation of $a^{(i)}_{j,k}$ as the proportion of points of $S$ that are also in $S_{(+,+)}$ (resp. $S_{(+,-)}$).
\end{definition}

For example, a weight has a $(+,+)$ activation equal to $0$ if and only if $S_{(+,+)} = \emptyset$. Based on $(+,+)$ and $(+,-)$ activations, we give a classification of unrecovered weights as follows:

\begin{table}[h!]
\renewcommand{\arraystretch}{1.2}
\centering
\setlength{\tabcolsep}{6pt}
\begin{tabular}{llcc}
\hline
\multicolumn{2}{c}{\textbf{Taxonomy of unrecovered weights}} & $\mathbf{(+,+)}$ & $\mathbf{(+,-)}$ \\ \hline

\multicolumn{1}{|l|}{\multirow{1}{*}{\textbf{Non-effective weights}}} 
&  & $0$ & $\geq 0$ \\
\hline

\multicolumn{1}{|l|}{\multirow{2}{*}{\textbf{Effective weights}}}
& unreachable weights & $>0$ & $0$ \\
& query-intensive weights & $>0$ & $>0$ \\ \hline
\end{tabular}
\end{table}

\FloatBarrier

\medskip
\noindent \textbf{Query-intensive weights.} 
\label{Query-intensive weights}
To recover a weight, we must hit some specific hyperplanes. The budget for the attack is limited and thus we cannot grid the whole $S$ with search lines. The targeted search, one of the few steps in the attack which we do not improve, has its limitations too. Therefore it is expected that we miss some effective weights. We refer to those weights as query-intensive weights. The number of critical points of a neuron we gather is not clearly correlated to how close it is to being always on/off. Despite this lack of correlation at the neuron level, we observe that in general increasing the number of critical points used reduces significantly the number of query intensive weights, see Table~\ref{table:taxonomy_final}. 

\begin{table}[htb!]
\vspace{-0.7cm}
\caption{
    Taxonomy of unrecovered weights in all extracted models, under the assumption that all preceding layers have been perfectly extracted, for different attack costs. Model I having a much larger input size, we ran into memory limitations when performing the attack with $6{,}000$ critical points.}
\resizebox{\hsize}{!}{
\centering
\setlength{\tabcolsep}{6pt}
\begin{tabular}{l c c c c c}
\hline
\textbf{Models} & \textbf{I} & \multicolumn{2}{c}{\textbf{II}} & \multicolumn{2}{c}{\textbf{III}} \\
Critical points used &  $3{,}000$ & $3{,}000$ & $6{,}000$ & $3{,}000$  & $6{,}000$\\

\hline
non-effective weights & 98   & 104  & 102 &130 &96 \\
\hdashline
unreachable weights & 5   & 8  &8 & 0 &0  \\
query-intensive weights   & 18  & 11 &5  & 25 &7  \\
\hline
\textbf{Model Coverage} & $91.75\%$ & $74.12\%$ & $74.78\%$ & 
$71.35\%$ &$79.53\%$\\
\hline
\end{tabular}
\label{table:taxonomy_final}
\vspace{-0.5cm}
}
\end{table}

\medskip\noindent \textbf{Unreachable weights.}
\label{unreachable weights}
Suppose that all critical points of a neuron \(\eta_j^{(i)}\) lie on the inactive side of a neuron \(\eta_k^{(i-1)}\). In this case, the contribution of \(\eta_k^{(i-1)}\) to the output of \(\eta_j^{(i)}\) is always zero during extraction, leaving us with no information about the corresponding weight \(a_{j,k}^{(i)}\). 
If no region of $S$ exists where both neurons are active (as illustrated in Fig.~\ref{FakeDead}), then $a_{j,k}^{(i)}$ is never activated and can be safely treated as non-effective. However, if a region does exist where both neurons are active (depicted in grey in Fig.~\ref{RealDead}), then $a_{j,k}^{(i)}$ is effective, contributes to the output of the network for inputs in this region, and should ideally be recovered. In this case, we say that $a_{j,k}^{(i)}$ is an \emph{unreachable} weight. Increasing the number of critical points does not diminish the number of unreachable weights, see Table~\ref{table:taxonomy_final}.

\begin{figure}[htb!]
    \centering
    \begin{subfigure}{0.48\textwidth}
        \centering
        \begin{tikzpicture}
        \draw (0,0) rectangle (4,3);
        
        \draw[thick, red]  (1,3) -- (1.2,1.5) -- (0.7,0.5) -- (0,0.2);
        \node[left] at (1.5,0.5) {\textcolor{red}{$\eta_j^{(i)}$}};
        
        \draw[thick , black] (3,3) -- (2.5, 1.5) -- (4,0.5);
        \node[right] at (2.7,0.5) {\textcolor{black}{$\eta_k^{(i-1)}$}};
    
        \node[left] at (2.5, 1.5) {\textcolor{black}{\large {--}}};
        \node[right] at (2.5, 1.53) {\textcolor{black}{\large{+}}};
    
        \node[left] at (1.2,1.53) {\textcolor{red}{\large{+}}};
        \node[right] at (1.3,1.5) {\textcolor{red}{\large{--}}};
        \end{tikzpicture}
        \caption{$a_{j,k}^{(i)}$ is non-effective.}
        \label{FakeDead}
    \end{subfigure}
    \hfill
    \begin{subfigure}{0.48\textwidth}
        \centering
        \begin{tikzpicture}
        \draw (0,0) rectangle (4,3);
        \draw [fill = black, fill opacity=.3] (3,3) -- (2.5, 1.5) -- (4,0.5) -- (4,3) -- cycle;
        \draw[thick, red]  (1,3) -- (1.2,1.5) -- (0.7,0.5) -- (0,0.2);
        \node[left] at (1.5,0.5) {\textcolor{red}{$\eta_j^{(i)}$}};
        
        \draw[thick , black] (3,3) -- (2.5, 1.5) -- (4,0.5);
        \node[right] at (2.7,0.5) {\textcolor{black}{$\eta_k^{(i-1)}$}};
    
        \node[left] at (2.5, 1.5) {\textcolor{black}{\large {--}}};
        \node[right] at (2.5, 1.53) {\textcolor{black}{\large{+}}};
    
        \node[left] at (1.2,1.5) {\textcolor{red}{\large{--}}};
        \node[right] at (1.3,1.53) {\textcolor{red}{\large{+}}};
        \end{tikzpicture}
        \caption{$a_{j,k}^{(i)}$ is unreachable (and so effective).}
        \label{RealDead}
    \end{subfigure}
    \vspace{0.3cm}
    \caption{The arrangement of hyperplanes can cause some effective weights to not be recoverable by the current extraction strategies. We call them unreachable.}
\end{figure}

\section{Experimental Results}
\label{experiments}
All our experiments were conducted on an Intel Core i7-14700F CPU using networks trained on either the {\tt MNIST} or {\tt CIFAR-10} datasets, two widely used benchmarking datasets in computer vision. They have also been commonly used in related works~\cite{EC:CCHRSS24,C:CarJagMir20,Hannah} to evaluate model extraction methods. {\tt MNIST} consists of 28 $\times$ 28 pixel grayscale images of handwritten digits, divided into ten classes (``0'' through ``9'') \cite{NIPS1989_53c3bce6}. In contrast, {\tt CIFAR-10} contains 32 $\times$ 32 pixel RGB images of real-world objects across ten classes (e.g., airplane, automobile, bird, cat, deer, dog, frog, horse, ship, truck) \cite{krizhevsky2009cifar}. 
The search space is the box centered at the origin of radius $10^3$ so as to keep the effective architecture close to the original.

\subsection{Comparison with the State-of-the-Art}
On a $784\!-\!16^{(8)}\!-\!1$ \texttt{MNIST} network, the state-of-the-art extraction of~\cite{Hannah} fails to recover any hidden layer beyond the third, even under the optimal assumption that all previous layers have been perfectly extracted. On the same architecture, with the same assumptions, our attack recovers at least $93\%$ of the effective weights in each deeper layer. We further evaluate our method on other architectures and observe similar performance. The results are given in Table~\ref{table:result_summary_table}.

\begin{table*}[htb!]
\vspace{-0.7cm}
\caption{
    Results demonstrating the weight extraction of layers deeper than those recovered by state-of-the-art methods in \cite{Hannah}. Each layer is attacked independently using $3,000$ critical points. Layer 9 is linear and can therefore be trivially recovered using previously issued queries.
}
\resizebox{\hsize}{!}{
\begin{tabular}{cllllllllllc}
\toprule
\textbf{Model} & \textbf{Signature}  & \textbf{L1}$^*$ & \textbf{L2} & \textbf{L3} & \textbf{L4} & \textbf{L5} & \textbf{L6} &\textbf{ L7} & \textbf{L8} & \textbf{L9} & \textbf{Entire} \\
& {\bf Extraction} & & & & & & & & & &{\bf Model} \\
\midrule
\multicolumn{12}{c}{\textbf{Ours}} \\
\midrule

 \textbf{{\tt CIFAR-10}} & Effective Neurons & $8$ & $8$ & $8$ & $8$ & $6$ & $8$ & $8$ & $6$ & $10$ & - \\
 $3072\text{-}8^{(8)}\text{-}10$ & Eff.$^\ddagger$ Weights Recovered & $100\%$ & $100\%$ & $90.32\%$ & $98.36\%$ & $89.74\%$ & $88.10\%$ & $94.92\%$ & $90.70\%$ & $100\%$ & $99.91\%$ \\
 (I)  & Time & $56$m$48$s & $1$h$33$m & $1$h$42$m & $4$h$43$m & $4$h$25$m & $5$h$2$m & $6$h$53$m & $2$h$23$m & $0.01$s & $27$h$38$m \\
   & Queries & $2^{23.55}$ & $2^{26.15}$ & $2^{26.15}$ & $2^{27.10}$ & $2^{27.04}$ & $2^{27.16}$ & $2^{27.56}$ & $2^{26.45}$ & $0$ & $2^{29.72}$\\[0.5ex]  \hdashline 
\textbf{{\tt MNIST}}    & Effective Neurons & $8$ & $7$ & $8$ &  $7$ & $6$ & $7$ & $8$ & $6$ & $1$ & - \\
$784\text{-}8^{(8)}\text{-}1$ & Eff. Weights Recovered & $100\%$ & $100\%$ & $100\%$ &  $100\%$ & $85.37\%$ & $90.91\%$ & $88.24\%$ & $92.86\%$ & $100\%$ & $99.71\%$ \\
(II)    & Time & $3$m$24$s & $7$m$5$s & $7$m$37$s & $10$m$41$s & $11$m$46$s & $28$m$29$s & $20$m$38$s & $36$m$46$s & $0.01$s & $2$h$6$m \\
& Queries  & $2^{21.59}$ & $2^{24.21}$ & $2^{24.22}$ & $2^{24.24}$ & $2^{24.26}$ & $2^{25.35}$ & $2^{24.96}$ & $2^{25.61}$ & $0$ &  $2^{27.64}$\\[0.5ex]  \hdashline    
\textbf{{\tt MNIST}}    & Effective Neurons & $16$ & $16$ & $16$ & $16$ & $16$ & $15$ & $16$ & $12$ & $1$ & - \\
$784\text{-}16^{(8)}\text{-}1$ & Eff. Weights Recovered & $100\%$ & $100\%$ & $100\%$ & $99.61\%$ & $99.21\%$ & $97.01\%$ & $98.64\%$ & $93.78\%$ & $100\%$ & $99.82\%$ \\
(III)   & Time & $6$m$17$s & $26$m$13$s & $28$m$57$s & $39$m$33$s & $1$h$39$m & $2$h$13$m & $1$h$48$m & $3$h$56$m & $0.01$s & $11$h$17$m \\
& Queries & $2^{21.60}$ & $2^{24.22}$ & $2^{24.24}$ & $2^{24.66}$ & $2^{25.63}$ & $2^{26.02}$ & $2^{25.78}$ & $2^{26.93}$ & $0$ & $2^{28.48}$\\ \hline
\multicolumn{12}{c}{\textbf{Foerster et al.~\cite{Hannah}}} \\ 
\midrule
 \textbf{{\tt MNIST}}    & All Weights & $100\%$ & $100\%$ & - & $\geq 37.50\%^\dagger$ & - & - & - & $\geq 0\%^\dagger$ & $100\%$ & - \\
  $784\text{-}16^{(8)}\text{-}1$   & Time & $2$h$46$m & $7$m$50$s & - & \textgreater 36 h & - & - & - & \textgreater 36 h & $0.01$s & - \\
 & Queries & $2^{22.36}$ & $2^{19.01}$ & - & - & - & - & - & - & $0$ & - \\ 
    \bottomrule
\end{tabular}
}

\flushleft{\small  
    $^*$L stands for layer. \textbf{L1} is attacked with $500$ critical points.
    $^\ddagger$'Eff.' is short for effective.
    $^\dagger$Foerster et al. recover 6/16 neurons on layer 4 and 0/16 neurons on layer 8, running the attack for 36 hours before stopping.
}
\label{table:result_summary_table}
\end{table*}

The taxonomy of these extraction results, as well as those obtained using $6{,}000$ critical points, is given above in Table \ref{table:taxonomy_final}. These results confirm our expectations. First, despite a large $S$, we observe a large amount of non-effective weights, about as much as one or two hidden layers. This shows that previously used metrics in the literature, notably the numbers of total weights recovered, can be misleading. Second, increasing the attack cost has the intended effects on the quality of the extraction: query-intensive weights diminish unlike unreachable weights. By doubling the number of critical points, Model II's coverage barely changed from 74.12\% to 74.78\% while Model III's, which does not have any unreachable active weights, changed more substantially from 71.35\% to 79.53\%. For both models only 0.9\% additional weights are recovered.

\subsection{Results of the End-to-End Attack}

We give in Table \ref{end-to-end_results} the results for the end-to-end extraction. For sign recovery, we employ the extension strategy. In model III, the extraction process terminates at layer 7, where only 12 of 16 neuron signatures are recovered with low precision, thereby preventing sign extraction and the subsequent signature extraction at layer 8. Note how, contrary to the previous extraction, the proportion of effective weights recovered plunges from layer five onwards.

\begin{table*}[htb!]
\vspace{-0.7cm}
\caption{
    End-to-end attack results using $3,000$ critical points for all but the first layer. Layer 9 is linear and can therefore be trivially recovered using previously made queries.
}
\resizebox{\hsize}{!}{
\begin{tabular}{clcccccccccc}
\toprule
$\quad$\textbf{Model} & \textbf{End-to-end} & \textbf{L1}$^*$ & \textbf{L2} & \textbf{L3} & \textbf{L4} & \textbf{L5} & \textbf{L6} &\textbf{ L7} & \textbf{L8} & \textbf{L9} & \textbf{Entire} \\ 
& {\bf Extraction} & & & & & & & & & & {\bf Model}  \\
\midrule
    & Effective Neurons & $8$ & $8$ & $8$ & $8$ & $6$ & $8$ & $8$ & $6$ & $10$ & - \\
    & Layer coverage & $100\%$ & $100\%$ & $98.97\%$ & $75.11\%$ & $98.43\%$ & $95.06\%$ & $97.12\%$ & $99.06\%$ & $100\%$ & $73.47\%$\\
    \textbf{{\tt CIFAR-10}} & Eff.$^\ddagger$ Weights Recovered  & $100\%$ & $100\%$ & $90.32\%$ & $91.67\%$ & $92.86\%$ & $90.48\%$ & $94.92\%$ & $93.02\%$ & $100\%$ & $99.90\%$ \\
    $3072\text{-}8^{(8)}\text{-}10$  & Time for Signature & $1$h$14$m & $1$h$35$m & $1$h$38$m & $2$h$38$m & $5$h$35$m & $3$h$19$m & $1$h$41$m & $1$h$38$m & $0.01$s & $19$h$18$m\\
    (I) & Queries for Signature & $2^{24.56}$ & $2^{26.15}$ & $2^{26.15}$ & $2^{26.58}$ & $2^{27.4}$ & $2^{26.76}$ & $2^{26.18}$ & $2^{26.22}$ & $0$ & $2^{29.42}$\\
    & Time for Signs & $1.69$s & $21$s & $43$s & $2$m$1$s & $2$m$25$s & $2$m$43$s & $2$m$30$s & $1$m$51$s & $0.01$s & $12$m$36$s\\
    & Queries for Signs & $2^{24.17}$ & $2^{22.85}$ & $2^{24.19}$ & $2^{25.52}$ & $2^{19.87}$ & $2^{19.64}$ & $2^{26.83}$ & $2^{19}$ & $0$ & $2^{27.68}$\\[0.5ex]  
\midrule
    & Effective Neurons & $8$ & $7$ & $8$ &  $7$ & $6$ & $7$ & $8$ & $6$ & $1$ & - \\
    & Layer coverage \rule{0pt}{2ex} & $100\%$ & $100\%$ & $100\%$ & $100\%$ & $93.75\%$ & $93.75\%$ & $76.21\%$ & $88.50\%$ & $100\%$ & $74.82\%$ \\
    \textbf{{\tt MNIST}} & Eff. Weights Recovered & $100\%$ & $100\%$ & $100\%$ &  $100\%$ & $85.37\%$ & $87.5\%$ & $86.27\%$ & $92.86\%$ & $100\%$ & $99.68\%$ \\
    $784\text{-}8^{(8)}\text{-}1$  & Time for Signature & $12$m$8$s & $7$m$5$s & $7$m$46$s & $11$m$13$s & $13$m$23$s & $10$m$21$s & $22$m$7$s & $19$m$48$s & $0.01$s & $1$h$43$m\\
    (II) & Queries for Signature & $2^{24.36}$ & $2^{24.21}$ & $2^{24.22}$ & $2^{24.24}$ & $2^{24.25}$ & $2^{24.27}$ & $2^{24.97}$ & $2^{24.78}$ & $0$ & $2^{27.44}$ \\
    & Time for Signs & $0.19$s & $0.85$s & $47$s & $1.46$s & $1.33$s & $1$m$35$s & $1$m$24$s & $12$s & $0.01$s & $4$m$2$s\\
    & Queries for Signs & $2^{20.23}$ & $2^{16.79}$ & $2^{24.26}$ & $2^{20.32}$ & $2^{15.68}$ & $2^{25.24}$ & $2^{25.48}$ & $2^{22.4}$ & $0$ & $2^{26.77}$\\[0.5ex]  
\midrule
    & Effective Neurons & $16$ & $16$ & $16$ & $16$ & $16$ & $15$ & $16$ & - & - & - \\
    & Layer coverage \rule{0pt}{2ex} & $100\%$ & $100\%$ & $100\%$ & $98.78\%$ & $88.28\%$ & $87.69\%$ & $28.09\%$ & - & - & -\\
    \textbf{{\tt MNIST}} & Eff. Weights Recovered & $100\%$ & $100\%$ & $100\%$ & $99.61\%$ & $91.63\%$ & $89.57\%$ & $76.39\%$ & - & - & - \\
    $784\text{-}16^{(8)}\text{-}1$  & Time for Signature & $21$m$35$s & $24$m$9$s & $24$m$47$s & $39$m$16$s & $1$h$59$m & $1$h$11$m & $2$h$44$m & - & - & -\\
    (III) & Queries for Signature & $2^{23.69}$ & $2^{24.31}$ & $2^{24.27}$ & $2^{24.69}$ & $2^{26.01}$ & $2^{25.1}$ & $2^{25.63}$ & - & - & -\\
    & Time for Signs & $0.73$s & $7$m$44$s & $7$m$49$s & $5$m$32$s & $8$m$22$s & $7$m$11$s & $7$m$38$s & - & - & -\\
    & Queries for Signs & $2^{20.23}$ & $2^{24.7}$ & $2^{25.31}$ & $2^{25.64}$ & $2^{25.68}$ & $2^{26.4}$ & $2^{26.43}$ & - & - & -\\ 
    \bottomrule
\end{tabular}
}

\flushleft{\small  
    $^*$L stands for layer. \textbf{L1} is attacked with $500$ critical points.
}
\label{end-to-end_results}
\end{table*}

The sudden drop in coverage in layer 4 of Model I is caused by a weight with a $(+,+)$ activation of $23\%$ of an almost-always-on neuron ($98\%$ active). Either making more queries, or slightly reducing the search space $S$ should increase substantially the coverage. The drop in coverage in layer 7 of Model II is caused by a particularly active unreachable weight with a $(+,+)$ activation of $11\%$. The $(\epsilon, \delta)$ analysis of the extraction of Model I and II is given in Fig.~\ref{fig:coverage_point}. We reach a $(3.64\times 10^{-4},0.26)$ extraction for Model I and a $(2.96\times 10^{-5},0.25)$ extraction for Model II. The choice of $\delta$ values is guided by their respective model coverage. 

%% file: tikzplots/epsilon_delta.tex
\definecolor{darkgray176}{RGB}{176,176,176}
\definecolor{lightgray204}{RGB}{204,204,204}
\definecolor{pastelgreen}{rgb}{0.47, 0.87, 0.47}
\definecolor{pastelred}{rgb}{1.0, 0.41, 0.38}
\definecolor{RawSienna}{rgb}{0.83, 0.54, 0.35}

\resizebox{0.50\columnwidth}{!}{
\begin{tikzpicture}
    \begin{semilogyaxis}[
        width=0.8\textwidth,
        height=0.5\textwidth,
        legend cell align={left},
        legend style={font=\footnotesize, draw=lightgray204},
        tick align=outside,
        tick pos=left,
        x grid style={darkgray176},
        xlabel={1 - $\delta$},
        xmin=0, xmax=1,
        xtick style={color=black},
        xtick={0,0.1,0.2,0.3,0.4,0.5,0.6,0.7,0.8,0.9,1},
        xticklabels={0,0.1,0.2,0.3,0.4,0.5,0.6,0.7,0.8,0.9,1},
        y grid style={darkgray176},
        ylabel={$\epsilon$($\log_{10}$)},
        ymajorgrids,
        ymin=1e-8, ymax=1.5,
        yticklabels={1e-8,1e-6,1e-4,1e-2,1},
        ytick style={color=black}
    ]

    \addplot [mark=none, line width=1pt, pastelgreen] coordinates {
        ( 0.02 , 1.022239292092846 )
        ( 0.04 , 0.5038861548164633 )
        ( 0.06 , 0.31002896890489845 )
        ( 0.08 , 0.22863540407285815 )
        ( 0.1 , 0.18392026426132232 )
        ( 0.12 , 0.13452212630780103 )
        ( 0.14 , 0.10376892854238341 )
        ( 0.16 , 0.07656821721666444 )
        ( 0.18 , 0.056345907067737 )
        ( 0.2 , 0.0433832152402085 )
        ( 0.22 , 0.027052947299429912 )
        ( 0.24 , 0.012800450443700673 )
        ( 0.26 , 0.0018582789766969652 )
        ( 0.28 , 3.725087959595374e-05 )
        ( 0.3 , 1.940160083804007e-05 )
        ( 0.32 , 1.3798072144378336e-05 )
        ( 0.34 , 1.0832691392547217e-05 )
        ( 0.36 , 8.910765355186443e-06 )
        ( 0.38 , 7.525967157155579e-06 )
        ( 0.4 , 6.453236163100609e-06 )
        ( 0.42 , 5.592751489971074e-06 )
        ( 0.44 , 4.87743620549814e-06 )
        ( 0.46 , 4.278074808042015e-06 )
        ( 0.48 , 3.771406028390141e-06 )
        ( 0.5 , 3.337541978205096e-06 )
        ( 0.52 , 2.9632811917639226e-06 )
        ( 0.54 , 2.638841021068375e-06 )
        ( 0.56 , 2.35790428191727e-06 )
        ( 0.58 , 2.1130396167366854e-06 )
        ( 0.6 , 1.8981974349931937e-06 )
        ( 0.62 , 1.7077910475706465e-06 )
        ( 0.64 , 1.5380234440249545e-06 )
        ( 0.66 , 1.3858439982416463e-06 )
        ( 0.68 , 1.2488366963514943e-06 )
        ( 0.7000000000000001 , 1.123654666885605e-06 )
        ( 0.72 , 1.009444447413646e-06 )
        ( 0.74 , 9.040542576099196e-07 )
        ( 0.76 , 8.062444439905327e-07 )
        ( 0.78 , 7.155153006204942e-07 )
        ( 0.8 , 6.30990167044453e-07 )
        ( 0.8200000000000001 , 5.517700001623423e-07 )
        ( 0.84 , 4.771331169447451e-07 )
        ( 0.86 , 4.067942946189488e-07 )
        ( 0.88 , 3.3987138102476493e-07 )
        ( 0.9 , 2.768685724339845e-07 )
        ( 0.92 , 2.1703395820415015e-07 )
        ( 0.9400000000000001 , 1.6024632593260504e-07 )
        ( 0.96 , 1.0568684903238967e-07 )
        ( 0.98 , 5.2491525477755105e-08 )
    };
    \addlegendentry{Model \textbf{I}}
    \addplot [mark=none, line width=1pt, pastelred] coordinates {
        ( 0.02 , 0.2651936439761061 )
        ( 0.04 , 0.13013800853712482 )
        ( 0.06 , 0.08563258809617158 )
        ( 0.08 , 0.07559628473820816 )
        ( 0.1 , 0.046785929150406906 )
        ( 0.12 , 0.017970769449599685 )
        ( 0.14 , 0.001714022571459516 )
        ( 0.16 , 0.0009953868104327752 )
        ( 0.18 , 0.0009429706833153497 )
        ( 0.2 , 0.00019760024779590037 )
        ( 0.22 , 4.874569406423486e-05 )
        ( 0.24 , 3.380480546408625e-05 )
        ( 0.26 , 2.7408754987766e-05 )
        ( 0.28 , 2.3482220926232793e-05 )
        ( 0.3 , 2.061880822539849e-05 )
        ( 0.32 , 1.8257868377299113e-05 )
        ( 0.34 , 1.6026177080944484e-05 )
        ( 0.36 , 1.3402945315739573e-05 )
        ( 0.38 , 1.0397698818251937e-05 )
        ( 0.4 , 8.356156445591966e-06 )
        ( 0.42 , 6.938538897654146e-06 )
        ( 0.44 , 5.883955293139632e-06 )
        ( 0.46 , 5.059336690347309e-06 )
        ( 0.48 , 4.391705315088805e-06 )
        ( 0.5 , 3.839485014647683e-06 )
        ( 0.52 , 3.3803057784268238e-06 )
        ( 0.54 , 2.9996013966175997e-06 )
        ( 0.56 , 2.6803192174143657e-06 )
        ( 0.58 , 2.410795553939549e-06 )
        ( 0.6 , 2.17754076222791e-06 )
        ( 0.62 , 1.9738105640308577e-06 )
        ( 0.64 , 1.7944329537225773e-06 )
        ( 0.66 , 1.6341090071991496e-06 )
        ( 0.68 , 1.4881351974676053e-06 )
        ( 0.7000000000000001 , 1.3535086759531997e-06 )
        ( 0.72 , 1.2281863080439447e-06 )
        ( 0.74 , 1.1125851478015549e-06 )
        ( 0.76 , 1.00524704872076e-06 )
        ( 0.78 , 9.051628545804619e-07 )
        ( 0.8 , 8.10866594808871e-07 )
        ( 0.8200000000000001 , 7.21653419746238e-07 )
        ( 0.84 , 6.354784143228022e-07 )
        ( 0.86 , 5.521413848898066e-07 )
        ( 0.88 , 4.7072404369806566e-07 )
        ( 0.9 , 3.901989637775688e-07 )
        ( 0.92 , 3.111485205316917e-07 )
        ( 0.9400000000000001 , 2.3265258364972653e-07 )
        ( 0.96 , 1.5476092182466415e-07 )
        ( 0.98 , 7.718984418784928e-08 )
    };
    \addlegendentry{Model \textbf{II}}

    \addplot [mark=none, line width=1pt, black, densely dashed] coordinates {
        ( 0.2518 , 2.95899965456409e-05 )
        ( 0.2518 , 1e-8 )
    };
    
    \addplot [mark=none, line width=1pt, black, densely dashed] coordinates {
        ( 0.2653 , 0.00036357050472036833 )
        ( 0.2653 , 1e-8 )
    };
    \addlegendentry{Model Coverage}
    
    \end{semilogyaxis}
\end{tikzpicture}
}

%% file: content/conclusion.tex
\section{Conclusion} \label{sec:conclusion}

In this work, we carried out an in-depth analysis of the extraction approach
introduced in~\cite{C:CarJagMir20}, reused in several follow-up works~\cite{EC:CCHRSS24,cryptoeprint:2024/1580,Hannah}, and identified critical limitations that prevented recovering weights beyond the third hidden layer. We then proposed improvements to the sign-recovery process by cleverly combining two methods introduced from Canales-Martínez et al.~\cite{EC:CCHRSS24}. Finally, we introduced several techniques that together improve the numerical precision of the extracted signature. Taken as a whole, these advances permit, for the first time, polynomial end-to-end extraction of neural networks of significant depth and open the way to practical attacks on larger architectures.

\medskip
\noindent \textbf{Limitations}.
The main limitation we faced to going much deeper with the non end-to-end attack is the influence of deeper layers when targetting small-width networks, and the cost of the attack when targetting large-width networks. For the end-to-end attack, floating-point imprecisions remain the main limitation. Addressing these issues is a promising direction for future work.

%% file: content/appendix.tex
\section{Example of a Small Neural Network}
\label{example_network}
We illustrate here the definitions introduced in \cref{definitions}. Consider the $2-3-2-1$ neural network given by the following layers $A^{(i)}$, $b^{(i)}$:
\[
A^{(1)} = \begin{pmatrix}
1 & 1 \\
1 & -1 \\
0.5 & 2\\
\end{pmatrix},  b^{(1)} = \begin{pmatrix}
    1\\
    -2\\
    -5\\
\end{pmatrix}
\]
\[
A^{(2)} = \begin{pmatrix}
{\color{blue} 2} & {\color{blue}-4} & {\color{blue}-1}\\
{\color{red}-3} & {\color{red}4} & {\color{red}5}\\
\end{pmatrix},  b^{(2)} = \begin{pmatrix}
    {\color{blue}-2}\\
    {\color{red}-3}\\
\end{pmatrix}
\]
\[
A^{(3)} = \begin{pmatrix}
1 & -1\\
\end{pmatrix},  b^{(3)} = \begin{pmatrix}
    3
\end{pmatrix}
\]

The architecture of this neural network is depicted in Fig.~\ref{fig:network}. Let the input to the neural network be $v=(x,y)$. The neurons in the first layer output $(\eta^{(1)}_1(v),\eta^{(2)}_1(v),\eta^{(3)}_1(v)) = (x+y+1, x-y-2, 0.5x+2y-5)$. Depending on $(x,y)$, the ReLU activation function $\sigma$ might set some entries to $0$ before passing them to the second layer. For example, when the input is $v = (2,2)$, we have: $F^{(1)}(v) = (5,-2,0)$,  $\sigma \circ F^{(1)}(v) = (5,0,0)$, $F^{(2)}(v) = 5*(2, -3) + 0*(-4,4) + 0*(-1, 5) + (-2, -3) = (8, -18)$, $\sigma \circ F^{(2)}(v) = (8,0)$, and finally $f(v) = F^{(3)}(v) = 8*1+0*(-1)+3 = 11$. Since $F^{(1)}(v) = (5,-2,0)$, it follows that $v$ is a critical point of $\eta_3^{(1)}$.

This network can also be represented as a collection of polytopes (see Fig.~\ref{fig:Second-Layer Polytopes}). As can be seen in this figure, critical hyperplanes from the first layer make three straight lines, represented as dotted grey lines: $\eta^{(1)}_1:x+y+1$, $\eta^{(1)}_2:x-y-2$, and $\eta^{(1)}_3:0.5x+2y-5$. The hyperplanes of $\eta_1^{(2)}$ and $\eta_2^{(2)}$ are shown in their respective colors, with their active side marked by $+$ and their inactive side by $-$. They indeed break on the grey lines as critical hyperplanes on deeper layers break on all hyperplanes of previous layers: the input to $\eta_k^{(i)}$ is $F^{(i-1)}(v)$ rather than $v$ for $i = 2,\cdots,r$.

\begin{figure}[H]
\centering
\begin{tikzpicture}[
    scale=0.8, 
    every node/.style={scale=0.8}, 
    node distance={20mm}, 
    minimum size=0.8cm, 
    main/.style = {draw, circle}
]
\node[main] (1) at (0,3) {${x}_{1}$};
\node[main] (2) at (0,1) {${x}_{2}$};
\node[main] (3) at (2,4) {${\eta}^{(1)}_{1}$};
\node[main] (4) at (2,2) {${\eta}^{(1)}_{2}$};
\node[main] (5) at (2,0) {${\eta}^{(1)}_{3}$};
\node[main, blue] (6) at (4,3) {${\eta}^{(2)}_{1}$};
\node[main, red] (7) at (4,1) {${\eta}^{(2)}_{2}$};
\node[main] (8) at (6,2) {$y$};
\draw[->] (1) -- (3);
\draw[->] (1) -- (4);
\draw[->] (1) -- (5);
\draw[->] (2) -- (3);
\draw[->] (2) -- (4);
\draw[->] (2) -- (5);

\draw[->, blue] (3) -- (6) node[pos=0.2, right] {2};
\draw[->, red] (3) -- (7) node[pos=0.5, above] {-3};
\draw[->, blue] (4) -- (6) node[pos=0.5, left] {-4};
\draw[->, red] (4) -- (7) node[pos=0.5, left] {4};
\draw[->, blue] (5) -- (6) node[pos=0.5, below] {-1};
\draw[->, red] (5) -- (7) node[pos=0.1, right] {5};

\draw[->] (6) -- (8);
\draw[->] (7) -- (8);

\node[above of=1, node distance=2cm] (9) {Input Layer};
\node[right of=9, node distance =3cm] (10) {Hidden Layers};
\node[right of=10, node distance=3cm] (11) {Output Layer};
\end{tikzpicture}
\caption{Network representation of the neural network example, with the details of layer $2$ highlighted in colour.}
\label{fig:network}
\end{figure}
Let's compute $f_v$ for $v = (2,2)$ as an example:
\[
\Gamma_v = A^{(3)}\begin{pmatrix}
    1&0\\
    0&0\\
\end{pmatrix}A^{(2)}
\begin{pmatrix}
    1&0&0\\
    0&0&0\\
    0&0&1
\end{pmatrix}A^{(1)} = (2.5, 4)
\]
\[
\gamma_v = f(v)-\Gamma_v\cdot v = 11-13 = 2.
\]
\begin{figure}[h!]
    \centering
    \includegraphics[width=0.6\linewidth]{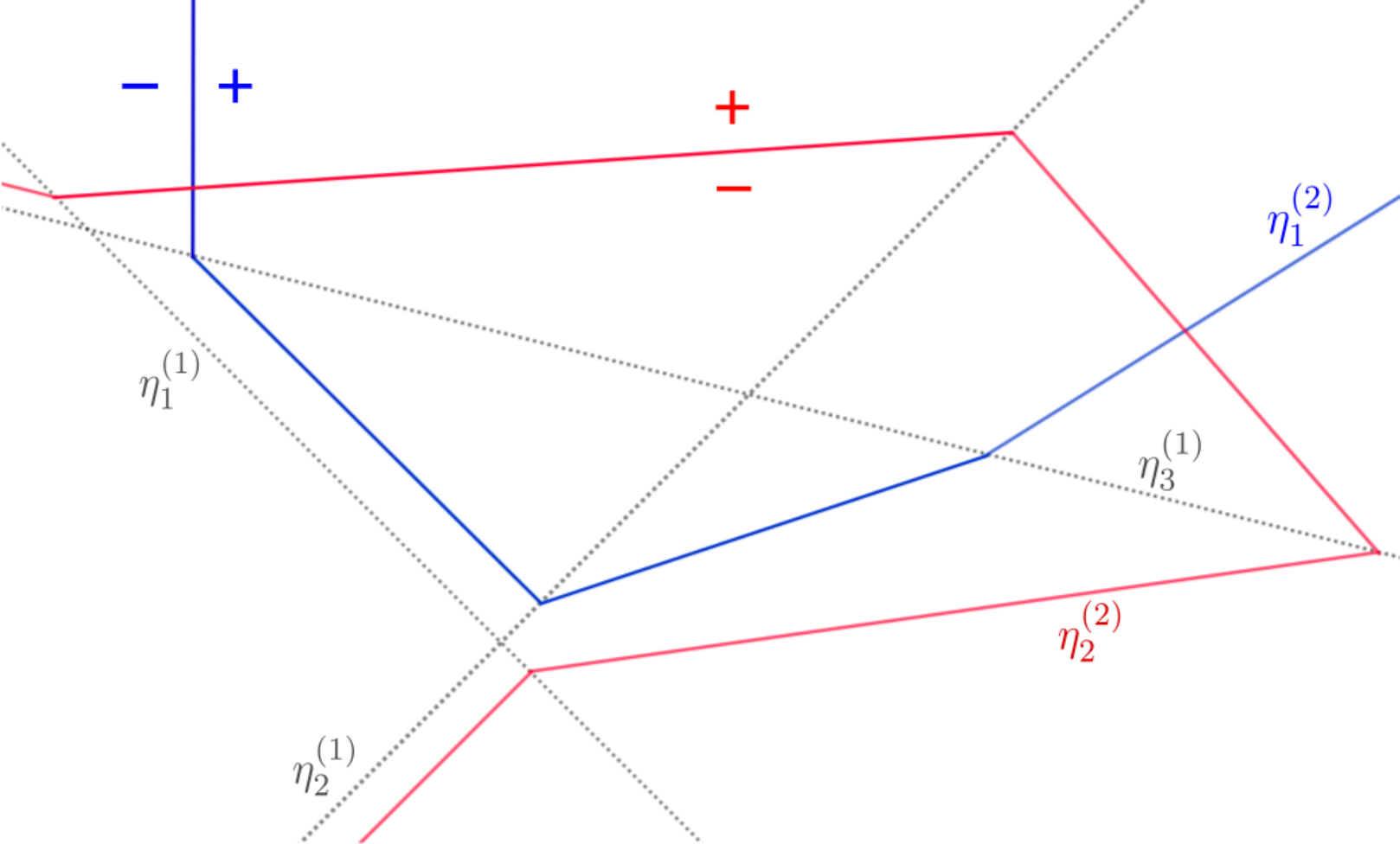}
    \caption{\label{fig:Second-Layer Polytopes}
    The neural network example partitions the input space into 18 distinct polytopes.}
\end{figure}



\section{Finding Critical Points}
\label{search for ccps}
The attack starts by searching for critical points of neurons on the layer we are targeting. Critical points are at the turning point of a ReLU and thus their left and right derivatives do not agree. We'll exploit this non-linearity to find critical points. We now describe the procedure developed in \cite{C:CarJagMir20} and in \cite{EC:CCHRSS24}. Very intuitive illustrations are given in both papers.

Suppose we perform the search between two points $a$ and $b$ on the search line. First we compute the derivatives of $f$ at $a$ towards $b$ and of $f$ at $b$ towards $a$, respectively $m_a$ and $m_b$. If they are the same then we know that there are no critical points on the search line. If they are not the same we know that there is at least a critical point on the search line between $a$ and $b$. We could directly reiterate the binary search on each half of the segment, etc. There is however a more efficient way to find the critical points. It relies on predicting where the critical point should be if it is the only one between $a$ and $b$, and checking if it is indeed there.

If there is a unique critical point, we expect to find it where the derivatives cross at $x^*  = a + \frac{f(b)-f(a)-m_b(b-a)}{a-b}$. Its expected value is $\hat{f}(x^* ) = f(a)+m_a\frac{f(b)-f(a)-m_b(b-a)}{a-b}$. If $\hat{f}(x^*) = f(x)$ then the critical point is unique, otherwise we keep dividing the segment with dichotomy. This faster algorithm can generate pathological errors, see (\cite{EC:CCHRSS24}, p.32) for more conditions to deal with them. In practice, we use \cite{C:CarJagMir20}'s strategy.

\section{Details of Signature Extraction from \cite{C:CarJagMir20}}

\subsection{Proof of Lemma 2}
\label{proof_lemma}
Let \( x \) be a critical point of layer \( i \in \{1, \ldots , r \} \) for neuron $\eta_k^{(i)}$ such that $x$ is not a critical point of any other neuron, and \( \Delta \in \mathbb{R}^{d_0} \). We have that:
\[
\partial^2_\Delta f (x) = c_k^{(i)} \big|A_k^{(i)}\cdot \Gamma_x^{(i-1)} \Delta \big|
\]

\begin{proof}
We take \(\epsilon_2\) sufficiently small so that \(x + \epsilon_2 \Delta\) and \(x -\epsilon_2 \Delta\) are in the neighbourhood of $x$. Thus \(x + \epsilon_2 \Delta\) and \(x -\epsilon_2 \Delta\) have the exact same activation pattern except for $\eta_k^{(i)}$. Let's write $G^{(i)}$ for the composition of layers starting from $i$, i.e. $f = G^{(i+1)}\circ F^{(i)}$.
\begin{align*}
\partial^2_\Delta f (x) & = \frac{1}{\epsilon_1}(f(x + \epsilon_2 \Delta) + f(x - \epsilon_2 \Delta) - 2 f(x)) \\
& = \frac{1}{\epsilon_1} \bigg( G^{(i+1)} \circ \sigma ( A^{(i)} F^{(i-1)}(x + \epsilon_2 \Delta) + b^{(i)} ) \\ 
&+ G^{(i+1)} \circ \sigma (A^{(i)} F^{(i-1)}(x - \epsilon_2 \Delta) + b^{(i)} ) \\
& -2 \, G^{(i+1)} \circ \sigma( A^{(i)} F^{(i-1)}(x) + b^{(i)} ) \bigg)  \\
& = \Omega_x^{(i)} \bigg( \sigma ( A^{(i)} F^{(i-1)}(x + \epsilon_2\Delta) + b^{(i)} ) \\ 
& + \sigma (A^{(i)} F^{(i-1)}(x - \epsilon_2 \Delta) + b^{(i)} ) \\
& -2 \, \sigma( A^{(i)} F^{(i-1)}(x) + b^{(i)} ) \bigg)
\end{align*}

Where \( \Omega_x^{(i)}\) is \( \frac{1}{\epsilon_1}(c_1^{(i)},\ldots, c_{d_i}^{(i)}) \) for some values $c_n \in \mathbb{R}$. Let us analyze \( C = \sigma ( A^{(i)} F_x^{(i-1)}(x + \epsilon_2 \Delta )  + b^{(i)} ) + \sigma (A^{(i)} F_x^{(i-1)}(x - \epsilon_2 \Delta ) + b^{(i)} ) - 2 \sigma (A^{(i)} F_x^{(i-1)}(x) + b^{(i)} ) \in \mathbb{R}^{d_i \times 1} \).

The coefficients \( j \neq k \) of \( C \) are zero. Indeed, since \( x \), \( x + \epsilon_2 \Delta \), and \( x - \epsilon_2 \Delta \) belong to the same activation region except for \( \eta_k^{(i)} \), the computation becomes affine, leading to zero. More precisely:

\begin{align*}
C[j] & \, = \, A_j^{(i)} \cdot [\Gamma_x^{(i-1)}(x + \epsilon_2 \Delta)+ \gamma_x^{(i-1)}]\\
& + A_j^{(i)} \cdot [\Gamma_x^{(i-1)}(x - \epsilon_2 \Delta) + \gamma_x^{(i-1)}]\\ 
&-A_j^{(i)}\cdot2[\Gamma_x^{(i-1)} x + \gamma_x^{(i-1)}]
\\ &= A_j^{(i)}\cdot \Gamma_x^{(i-1)} x + A_j^{(i)}\cdot  \epsilon_2\Gamma_x^{(i-1)} \Delta \\
&+ \,A_j^{(i)}\cdot \Gamma_x^{(i-1)} x \, - \,  A_j^{(i)}\cdot \epsilon_2\Gamma_x^{(i-1)} \Delta 
- A_j^{(i)}\cdot 2\Gamma_x^{(i-1)} x \\
&= 0 \quad  \text{when $\eta_j^{(i)}$ is active at $x$.}
\end{align*}

\( C[j] = 0 + 0 - 2 \times 0 = 0 \) when \( \eta_j^{(i)} \) is inactive at \( x \). 

For coefficient \( k \), since \( x  \) is a critical point of $\eta_k^{(i)}$, we have:

\[
A_k^{(i)}\cdot F^{(i-1)}(x) \, + \,  b^{(i)}_k = A_k^{(i)}\cdot [\Gamma_x^{(i-1)} x + \gamma_x^{(i-1)}] \, + \, b^{(i)}_k = 0
\]

and since \(A_k^{(i)}\cdot \Gamma_x^{(i-1)} \Delta \) is either positive or negative, we get:

\[
C[k] = \big|A_k^{(i)}\cdot \Gamma_x^{(i-1)} \Delta \big|
\]

Thus, by multiplying \(\Omega^{(i)}_x \) by \( C \), we obtain the expected result.
\end{proof}

Intuition behind the above proof is given in Fig.~\ref{intuition_extraction} below. 

\begin{figure}[htb!]
    \centering

    \begin{subfigure}{0.45\textwidth}
        \centering
        \begin{tikzpicture}[scale=1]
            \draw[thick] (0,0) -- (-2,-1);
            \draw[thick] (0,0) -- (1.5,1.5);
            \draw[thick, red, ->] (0,0) -- (-0.8,-0.4);
            \draw[thick, red, ->] (0,0) -- (0.6,0.6);
            \draw[very thick, blue, ->] (0,0) -- (-0.2, 0.2);
            \draw[thick, ->] (-3,-1.5) -- (-3,-1);
            \node[above] at (-3,-1) {$f(x)$};
            \draw[thick, ->] (-3,-1.5) -- (-2.5,-1.5);
            \node[right] at (-2.5,-1.5) {$x$};
            \filldraw[fill=black] (0,0) circle (2pt);
            \node[below] at (0,0) {$x^*$};
            \draw[dashed, white, thick] (-1.4,-1.4) -- (-1.5,-1.5); 
            \draw[dashed, thick] (0,0) -- (2,1);
        \end{tikzpicture}
        \caption{$\eta(x)\geq 0$ for $x\geq x^*$}
    \end{subfigure}
    \hfill
    \begin{subfigure}{0.45\textwidth}
        \centering
        \begin{tikzpicture}[scale=1]
            \draw[thick] (0,0) -- (-2,-1);
            \draw[thick] (0,0) -- (1.5,1.5);
            \draw[thick, red, ->] (0,0) -- (-0.8,-0.4);
            \draw[thick, red, ->] (0,0) -- (0.6,0.6);
            \filldraw[fill=black] (0,0) circle (2pt);
            \draw[very thick, blue, ->] (0,0) -- (-0.2, 0.2);
            \draw[thick, ->] (-3,-1.5) -- (-3,-1);
            \node[above] at (-3,-1) {$f(x)$};
            \draw[thick, ->] (-3,-1.5) -- (-2.5,-1.5);
            \node[right] at (-2.5,-1.5) {$x$};
            \filldraw[fill=black] (0,0) circle (2pt);
            \node[below] at (0,0) {$x^*$};
            \draw[dashed, thick] (0,0) -- (-1.5,-1.5);
            \draw[dashed, white, thick] (1.9, 0.9) -- (2,1);
        \end{tikzpicture}
        \caption{$\eta(x)\geq0$ for $x\leq x^*$}
    \end{subfigure}
    \vspace{0.3cm}
    \caption{Illustration of how a neuron’s activation change at a critical point creates a detectable 
    break in the network’s output, allowing the extraction of the neuron. The solid line shows the true output $f(x)$, while 
    the dotted line shows the hypothetical output if the neuron~$\eta$ were always 
    inactive. At the critical point $x^*$, $\eta$ flips its state, creating a change 
    in slope. Two cases can occur: (a) the neuron becomes active 
    ($\eta(x) \geq 0$ for $x \geq x^*$), in which case the solid line bends away from 
    the dotted line; or (b) the neuron becomes inactive ($\eta(x) \geq 0$ for $x \leq x^*$), 
    in which case the solid line aligns with the dotted line after $x^*$. 
    The slopes before and after $x^*$ (red arrows) differ by the absolute weight of $\eta$, 
    shown by the blue arrow. This indicates how much the neuron contributes, hence its weights, although an 
    additional step is required to determine the correct sign of the contribution.}
    \label{intuition_extraction}
\end{figure}

\subsection{Correlating the Weights' Signs}
\label{removing_absolute_values}
We know by Lemma \ref{lemma1} that \[
    \frac{\partial^2_{\Delta} f (x)}{\partial^2_{\Delta_0} f (x)} = \frac{\big|A_k^{(i)}\cdot(\Gamma_x^{(i-1)} \Delta)\big|}{\big|A_k^{(i)} \cdot(\Gamma_x^{(i-1)} \Delta_0)\big|} \coloneqq \frac{\big |\alpha|}{|\beta|}
\]
Now consider on one hand, 
\[\frac{\partial^2_{\Delta + \Delta_0} f (x)}{\partial^2_{\Delta_0} f(x)} = \frac{\big|\alpha + \beta \big|}{\big | \beta \big |}\]
And on the other, 
\[\frac{\partial^2_{\Delta} f (x)+\partial^2_{\Delta_0} f(x)}{\partial^2_{\Delta_0} f(x)} = \frac{\big|\alpha \big|+\big |\beta\big|}{\big|\beta \big|}\]
The two values above are equal if and only if $\alpha$ and $\beta$ have the same sign. This additional test allows us to remove the absolute values and correlate the sign of our result for each direction $\Delta$ we take to that of the first direction $\Delta_0$ taken. This is why the sign extraction only has to find one sign per neuron rather than one sign per weight.

\subsection{Solving for the Signature}
\label{extraction_detail}
By querying different directions $\{\Delta_m\}$ around $x$ a critical point of $\eta^{(i)}_k$, we can obtain a set of $\{y_m\}$ such that, \[y_m =\frac{A_k^{(i)}\cdot (\Gamma_x^{(i-1)} \Delta_m)}{c^{(i)}_k}, \] where $c^{(i)}_k = A_k^{(i)}\cdot(\Gamma_x^{(i-1)} \Delta_0)$.
To solve for $A_k^{(i)}$, we build the following system of equations:
\[
\begin{pmatrix}
\Gamma_x^{(i-1)} \Delta_0\\
    \Gamma_x^{(i-1)} \Delta_1\\
    \Gamma_x^{(i-1)} \Delta_2\\
    \vdots\\
    \Gamma_x^{(i-1)} \Delta_{d_{i-1}-1}
\end{pmatrix}\cdot \frac{1}{c^{(i)}_k}
\begin{pmatrix}
    a^{(i)}_{k,0}\\
    a^{(i)}_{k,1}\\
    a^{(i)}_{k,2}\\
    \vdots\\
     a^{(i)}_{k,d_{i-1}-1}
\end{pmatrix} = 
\begin{pmatrix}
    1\\
    y_1\\
    y_2\\
    \vdots\\
    y_{d_{i-1}-1}\\
\end{pmatrix}
\]

When extracting the first layer, the input of the layer is also the input to the network, giving us full control over it. We can use the input basis vectors $\{e_i\}$ as directions $\Delta$, allowing us to recover each entry of the signature independently from each equation. However, for deeper layers, we must first gather enough linearly independent conditions.

\section{Deeper Critical Points Merging with Components on the Target Layer}
\label{proof_noise2}
In this section, we show that a critical point $x_1$ of neuron $\eta_k^{(i)}$ on the target layer can merge with a critical point $x_2$ of a neuron in a deeper layer $i+t$ if $x_2$ causes $\eta_k^{(i)}$ to be the only active neuron on its layer, meaning that for some $a_k \in \mathbb{R}$,  
\[ \sigma\circ F^{(i)}(x_2) = (0, \dots, 0, a_k, 0, \dots, 0).\]
This indicates that a strict association between deeper critical points and incorrect components can lead to mislabelling. We note $y_1$ and $y_2$ the matrix rows extracted from $x_1$ and $x_2$ respectively. We therefore have,
\begin{align*}
\partial^2_\Delta f (x_1) &= (\Gamma_{x_1}^{(i-1)} \Delta) \cdot y_1\\
\partial^2_\Delta f (x_2) &= (\Gamma_{x_2}^{(i+t-1)} \Delta) \cdot y_2
\end{align*}
Let's write $A_{x_2}$ the matrix $A'_{x_2}\circ I^{(i)}_{x_2}\circ A^{(i)}$, where $A'_{x_2}=A^{(i+t-1)}\circ I^{(i+t-2)}_{x_2}\circ \dots \circ I^{(i+1)}_{x_2}\circ A^{(i+1)}$.
Thus, for $c \in \mathbb{R}$,
\begin{align*}
 (\Gamma_{x_2}^{(i+t-1)} \Delta) \cdot y_2 &= [(A_{x_2}\circ\Gamma_{x_2}^{(i-1)}) \Delta] \cdot y_2\\
&= (\Gamma_{x_2}^{(i-1)} \Delta) \cdot (A_{x_2}^{\top}y_2) \\
&= (\Gamma_{x_2}^{(i-1)} \Delta) \cdot [(A^{(i)\top}\circ I^{(i)\top}_{x_2}\circ A'^{\top}_{x_2})y_2]\\
&= (\Gamma_{x_2}^{(i-1)} \Delta) \cdot cy_1
\end{align*}

Let's explain the last equality. Our assumption that $x_2$ causes $\eta_k^{(i)}$ to be the only active neuron on its layer implies that $I^{(i)}$ is diagonal with all entries zero except for a one at the $k$-th diagonal position. We also have that $y_1$ is the $k$-th row of the matrix $A^{(i)}$. Therefore, $(A^{(i)\top}\circ I^{(i)\top}_{x_2}\circ A'^{\top}_{x_2})y_2=c\cdot y_1$ , where $c$ is a constant. The partial signatures $y_1$ and $c\cdot y_1$ extracted from $x_1$ and $x_2$ can be thus merged even though $x_2$ is on a deeper layer.

This case is very rare, even for small networks, but it might still happen. We encountered two such neurons in our experiments. The extracted component will have the right weights, but we might discard it because of a high $\tau$. We check for this phenomenon on neurons with a low $\tau$ (neurons we are almost sure are on the target layer): first we finish the extraction of those neurons, we could compute their sign using the neuron wiggle and if a critical point is on the inactive side of all those neurons, we do not count it towards its component's $\tau$.

On a side note, when this happens, the sign of all neurons on layer $i$ can be recovered immediately with no queries or computations as we have an input $x^*$ for which all neurons on the layer must be inactive (first order derivatives in $\mathcal{P}_{x^*}$ are equal to zero). We must have $\eta^{(i)}
_k \circ F^{(i-1)}( x^*) < 0$, for all neurons on the layer, so we choose the sign of each $\eta^{(i)}
_k$ to match this condition.

\section{Noise for Large Networks}
\label{Appendix Noise for large networks}
In this section, we show that recovered partial neurons from deeper layers correspond to very different mathematical expressions from neurons on the target layer, leading to the possibility of identifying deeper critical point at no cost using statistical arguments. Indeed, we can view each recovered weight as a random variable.\\
Let's start with a neuron $\eta^{(i+1)}_j$ on the $i+1$'th layer. The equation of its hyperplane is given by:
\[
A^{(i+1)}_j \cdot F^{(i)}(x)+b^{(i+1)}_j=0
\]
Which is, rewritten according to layer $i-1$, 
\[
A^{(i+1)}_j \cdot \left(I^{(i)} A^{(i)} F^{(i-1)}(x) + I^{(i)} b^{(i)} \right) + b^{(i+1)}_j = 0
\]
\[
\sum_{k=1}^{d_{i+1}} a^{(i+1)}_{j,k} \cdot \left[ I^{(i)} A^{(i)} F^{(i-1)}(x) \right]_k + \mathcal{B} = 0
\]
\[
\sum_{k=1}^{d_{i+1}} a^{(i+1)}_{j,k} \cdot \left( \sum_{l=1}^{d_i} I^{(i)}_{k,k} \cdot a^{(i)}_{k,l} \cdot [F^{(i-1)}(x)]_l \right) + \mathcal{B} = 0
\]
Expressing directly in terms of outputs of the extracted layer $i-1$:
\[
\sum_{l=1}^{d_i} \left( \sum_{k=1}^{d_{i+1}} a^{(i+1)}_{j,k} \cdot I^{(i)}_{k,k} \cdot a^{(i)}_{k,l} \right) [F^{(i-1)}(x)]_l + \mathcal{B} = 0
\]
and so, if it is not made inactive by a ReLU on layer $i-1$, the first weight extracted is:
\[
\sum_{k=1}^{d_{i+1}} a^{(i+1)}_{j,k} \cdot I^{(i)}_{k,k} \cdot a^{(i)}_{k,1}
\]
And therefore by a quick induction, the first extracted weight of a neuron $\eta_j^{(n)}$ when recovering layer $i$ correspond to all the active paths between the weight we think we are extracting and $\eta^{(n)}_j$:
\begin{align*}
&\sum_{k_n=1}^{d_{n}} a^{(n)}_{j,k_n} \cdot I^{(n-1)}_{k_n,k_n}\cdot(\sum_{k_{n-1}=1}^{d_{n-1}} a^{(n-1)}_{k_n,k_{n-1}} \cdot I^{(n-2)}_{k_{n-1},k_{n-1}}\\ &\cdot (\dots (\sum_{k_{i+1}=1}^{d_{i+1}} a^{(i+1)}_{k_{i+2},k_{i+1}} \cdot I^{(i)}_{k_{i+1},k_{i+1}} \cdot a^{(i)}_{k_{i+1},1})\dots))
\end{align*}

As an example, suppose for simplicity that every layer has width $d$, that for any given critical point we expect half of the neurons to be active, and that all the weights of the network have the same mean and variance. Let's find the distribution of our extracted $w^{(n)}$ when attempting to find layer $i$. Let's denote the set of paths as $P = \{(k_{i+1}, k_{i+2}, \dots, k_n, j)|\forall m, k_m \in \{1, \dots, d\}\}$. Then,
\[
w^{(n)} = \sum_{p\in P}a^{(i)}_{k_i+1, 1}\cdot\prod_{l = i}^{n-1} (a^{(l+1)}_{p_{l-i+2}, p_{l-i+1}}\cdot I^{(l)}_{p_{l-i+1},p_{l-i+1}}), \text{where } p_i \text{ is the } i\text{th entry of } p   
\]
Let's write $L$ for $n-i$, and $\mathcal{C}_L$ for $\frac{d^L}{2^L}$ then we have that,
\begin{align*}
   \mathbb{E}[w^{(n)}] &= d^L\cdot \mathbb{E}[w^{(i)}]\cdot (\mathbb{E}[w^{(i)}]\cdot \frac{1}{2})^{L} \\
   &= \mathcal{C}_L\cdot\mathbb{E}[w^{(i)}]^{L+1}
\end{align*}
and,
\begin{align*}
\mathbb{E}[(w^{(n)})^2] &= \mathbb{E}[\sum_{p\in P}\sum_{p'\in P} (a^{(i)}_{k_i+1, 1})^2\cdot\prod_{l = i}^{n-1} a^{(l+1)}_{p_{l-i+2}, p_{l-i+1}}\\ &\cdot I^{(l)}_{p_{l-i+1},p_{l-i+1}}\cdot a^{(l+1)}_{p'_{l-i+2}, p'_{l-i+1}}\cdot I^{(l)}_{p'_{l-i+1},p'_{l-i+1}}]
\end{align*}
Let's separate between $p=p'$ and $p \neq p'$. For $p=p'$, we have just as above, $d^{L}$ paths, $I^{(l)}_{p',p'} = I^{(l)}_{p,p}$ and $a^{(l)}_{p'} = a^{(l)}_{p}$. For $p \neq p'$, we have $d^{2L}-d^{L}$ possible path combinations, $I^{(l)}_{p',p'} \neq I^{(l)}_{p,p}$ and $a^{(l)}_{p'} \neq a^{(l)}_{p}$. By assumption $\mathbb{E}[I^{(l)}] = \frac{1}{2}$. We therefore get summing the $p=p'$ paths with the $p\neq p'$ paths, 
\[
\mathbb{E}[(w^{(n)})^2] = d^L\cdot \mathbb{E}[(w^{(i)})^2]\cdot (\mathbb{E}[(w^{(i)})^2]\cdot \frac{1}{2})^{L} + (d^{2L}-d^L)\cdot \mathbb{E}[w^{(i)}]^2\cdot (\mathbb{E}[w^{(i)}]^2\cdot \frac{1}{2^2})^L
\]
Let's simplify notation and set $\mathbb{E}[w^{(i)}] = \mu$, $\mathbb{V}(w^{(i)}) = \sigma^2$ for all $i$ using our assumption,  
\begin{align*}
\mathbb{E}[(w^{(n)})^2] &= d^L\cdot(\sigma^2+\mu^2)\cdot(\frac{\sigma^2+\mu^2}{2})^L+(d^{2L}-d^L)\cdot\mu^2\cdot(\frac{\mu}{2})^{2L}\\
&= \mathcal{C}_L\left[ (\mathcal{C}_L-1)\mu^{2L+2} +(\sigma^2+\mu^2)^{L+1}\right]
\end{align*}
and therefore, 
\begin{align*}
    \mathbb{V}(w^{(n)}) &= \mathbb{E}[(w^{(n)})^2]-\mathbb{E}[w^{(n)}]^2
    \\
    &= \mathcal{C}_L\left[ (\mathcal{C}_L-1)\mu^{2L+2} +(\sigma^2+\mu^2)^{L+1} \right]-\mathcal{C}_L^2 \mu^{2L+2} \\
    &= \mathcal{C}_L[(\sigma^2+\mu^2)^{L+1}-\mu^{2(L+1)}]
\end{align*}
Therefore for a given statistical distribution of $w^{(i)}$, the variance and the mean of weights recovered from a neuron on a deeper layer depends on how much deeper the neuron is compared to the target layer ($L = n-i$), as well as the dimension of each layer. In the attack, a partial neuron is extracted up to a scalar. Our experiments show that despite this scalar and underdetermined systems, variations in variance are enough to identify a large proportion of points on deeper layers, see Fig. \ref{fig:noise_stat_per_layer}. Being able to identify points on deeper layers leads to the speedup discussed in Section~\ref{subsec:filteralgo}, see Table~\ref{table:noisestats}.

\begin{table}[htb!]
    \centering
    \resizebox{\textwidth}{!}{%
        \begin{tabular}{l|r|r|r|r|r|r|r|r|r|r|r}
            L4 points kept (\%)     & 100 & 99 & 95 & 90 & 80 & 70 & 60 & 50 & \dots & 20 & 10 \\
            \hline
            L4 points proportion (\%) & 11.8 & 14.5 & 16.6 & 17.6 & 19.2 & 20.2 & 21.1 & 22.2 & \dots & 26.5 & 24.5 \\
            \hline
            Expected min speedup & 1 & 1.54 & 2.00 & 2.27 & 2.69 & 2.96 & 3.26 & 3.58 & \dots & 5.10 & 4.37\\ 
        \end{tabular}        
    }
    \vspace{0.3cm}\caption{Effectiveness of discarding partial weights with high variance in tackling noise when extracting layer 4 in a $784\!-\!256^{(16)}\!-\!1$ network with 100{,}000 points. Initially, points on layer 4 represent 11.7\% of all points. Imposing stricter upper bounds on the variance discards points predominantly on deeper layers, gradually increasing the proportion of points on layer 4. The expected minimum speed-up, $ (\text{L4 points kept/All points kept})^2$, is computed assuming the intersections and merging run in $\Omega(x^2)$ where $x$ is the number of critical points collected. For example, if the procedure keeps 50\% of the points on layer 4, while discarding 90\% of all points, we would need to collect twice the number of points to run the attack with the same number of points on the target layer. Therefore we would end up merging 80\% less points, meaning 5 times less. As the merging is at least quadratic, we would expect it to run at least 25 times faster. While trying to collect enough partial signatures to have a reasonable number of merges on the target layer, we ran into hardware limitations, making it impossible to compute empirical speedups (akin to \cite{cryptoeprint:2024/1580}). Given the rank deficiencies, $\Omega(x^2)$ is far from being tight and empirical speedup factors should be much greater.}
    \label{table:noisestats}
\end{table}

\begin{figure}[htb!]
    \centering
    \includegraphics[width=0.95\linewidth]{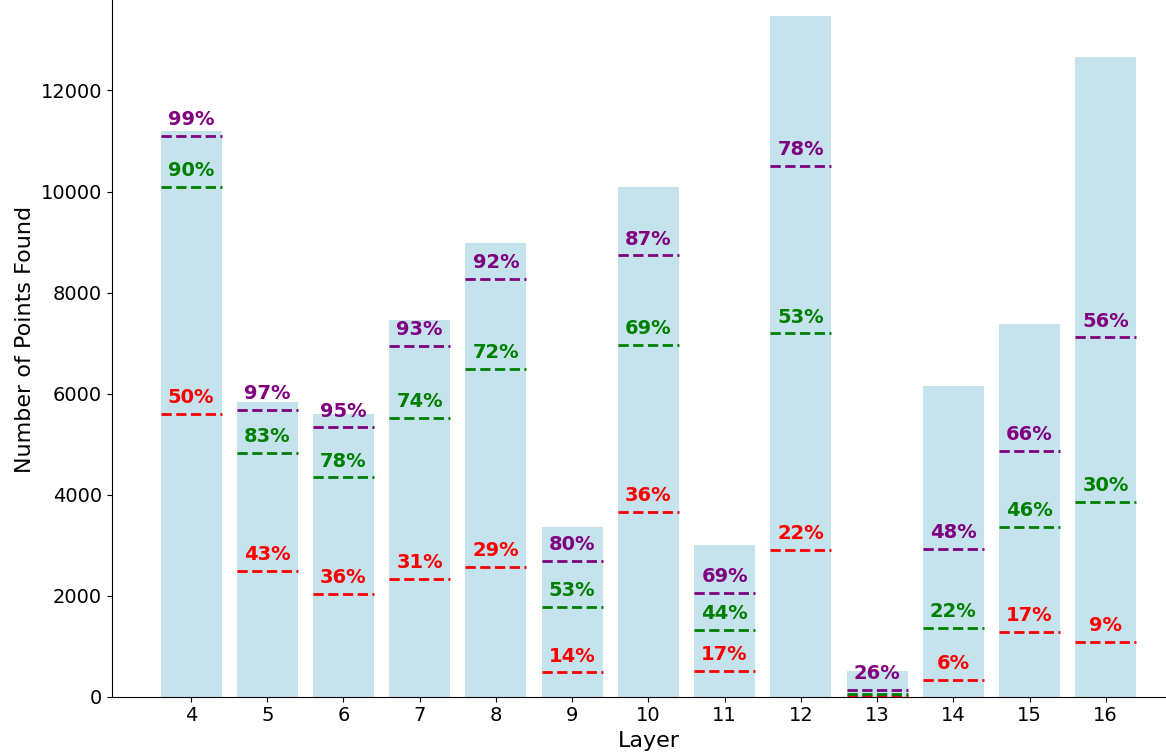}
    \caption{For each layer, the number of critical points we discard on a   $784\!-\!256^{(16)}\!-\!1$ network attacked with 100{,}000 points by imposing an upper bound on variance that would keep 99\%, 90\% and 50\% of points on layer 4, respectively in purple, green and red. As implied by the mathematical analysis, critical points from deeper layers are easier to identify: the percentage of points on deeper layers kept is always lower than that of layer 4.}
    \label{fig:noise_stat_per_layer}
\end{figure}

\FloatBarrier

\section{Impact of Scaling Factors on Precision}
\label{sclaing_factor_precision}
In this section, we give an example showing how large disparities among scaling factors degrade the precision of signature recovery. Without normalization, the scaling factors for layer 2 in a $784-8^{(8)}-1$ {\tt MNIST} model are
\[
(161.71, 32.11, 23.14, 1.07, 0.81, 36.64, 0.56, 0)
\]
where values are rounded to two decimal places for simplicity. The ratio between the largest factor (161.71) and the smallest factor (0.56) is 288.77.

Next, we can identify a critical point for the 8th neuron in layer 3, for which the recovered weights are imprecise. The true weights for this neuron are
\[
(0.51, 0.47, -0.55, 0.37, -0.28, 0.63, -0.23, -0.39)
\]
with relative ratios (normalized to the first element)

\[
(1, 0.92, -1.08, 0.73, -0.56, 1.25, -0.46, -0.77)\tag{a}\label{eq:ratio_a}
\]

Without normalization, the recovered weights are
\[
(-6.33\times 10^{-4}, -3.04\times 10^{-3}, 4.99\times 10^{-3}, -7.29\times 10^{-2}, 6.85\times 10^{-2}, -3.46\times 10^{-3}, 0, 0)
\]
After multiplication by the scaling factors, they become
\[
(-1.02\times 10^{-1}, -9.77\times 10^{-2}, 1.16\times 10^{-1}, -7.79\times 10^{-2}, 5.56\times 10^{-2}, -1.27\times 10^{-1}, 0, 0)
\]
with relative ratios (normalized to the first element)
\[
(1, 0.95, -1.13, 0.76, -0.54, 1.24, 0, 0)\tag{b}\label{eq:ratio_b}
\]

After applying normalization, the recovered weights are
\[
(-6.34\times 10^{-4}, -2.95\times 10^{-3}, 4.78\times 10^{-3}, -7.04\times 10^{-2}, 7.12\times 10^{-2}, -3.51\times 10^{-3}, 0, 0)
\]
After multiplication by the scaling factors, they become
\[
(-1.03\times 10^{-1}, -9.46\times 10^{-2}, 1.11\times 10^{-1}, -7.52\times 10^{-2}, 5.78\times 10^{-2}, -1.28\times 10^{-1}, 0, 0)
\]
with relative ratios (normalized to the first element)
\[
(1, 0.92, -1.08, 0.73, -0.56, 1.25, 0, 0)\tag{c}\label{eq:ratio_c}
\]

With normalization, the recovered weights exhibit perfect relative ratios (\ref{eq:ratio_c}), matching the true ratios (\ref{eq:ratio_a}). In contrast, without normalization, the recovered weights are imprecise, as reflected in the relative ratios (\ref{eq:ratio_b}).

\section{Increasing the Rank of SOE}
\label{SOE rank}
In this section, we analyze the rank of our improved SOE system, which stacks the SOE systems of multiple points. Consider $n$ points $\{x_j\}_{j\in[1,n]}$ for sign extraction on layer $i$. These points share the same activation pattern on later layers, denoted as $G_{x_j}^{(i+1)}$ (simplified as $G^{(i+1)}$), but trigger many different activation patterns for $F_{x_j}^{(i-1)}$. For each point $x_j$, we can construct a SOE system:
\[
\{G^{(i+1)}I^{(i)}A^{(i)}F_{x_j}^{(i-1)}\Delta_{k_j} = f(x_j+\Delta_{k_j})-f(x_j)\}_{k_j}
\]
by taking multiple input directions $\Delta_{k_j}$, where our unknown is $G^{(i+1)}I^{(i)}$. Stacking these systems yields a larger SOE system:
\[
G^{(i+1)}I^{(i)}
\begin{pmatrix}
    \{A^{(i)}F_{x_1}^{(i-1)}\Delta_{k_1}\}_{k_1}\\
    \{A^{(i)}F_{x_2}^{(i-1)}\Delta_{k_2}\}_{k_2}\\
    \vdots\\
    \{A^{(i)}F_{x_n}^{(i-1)}\Delta_{k_n}\}_{k_n}
\end{pmatrix}=
\begin{pmatrix}
    B_1\\
    B_2\\
    \vdots\\
    B_n
\end{pmatrix}=
\begin{pmatrix}
    \{f(x_1+\Delta_{k_1})-f(x_1)\}_{k_1}\\
    \{f(x_2+\Delta_{k_2})-f(x_2)\}_{k_2}\\
    \vdots\\
    \{f(x_n+\Delta_{k_n})-f(x_n)\}_{k_n}\\
\end{pmatrix}
\]
where each block matrix $\{A^{(i)}F_{x_j}^{(i-1)}\Delta_{k_j}\}_{k_j}$ is denoted as $B_j$. 

In this stacked SOE system, our focus is on the rank of
\[
C=\begin{pmatrix}
    B_1\\
    B_2\\
    \vdots\\
    B_n
\end{pmatrix}
\] 
as rank deficiency determines how many inactive neurons with high confidence must be recovered using neuron wiggle.

First, we give a trivial bound for $\text{rank}(C)$. Since $C$ contains all rows of $\{B_j\}_{j\in[1,n]}$, $\text{rank}(C)$ is at least the largest rank among the individual blocks: $\max_{1\leq j\leq n} \text{rank}(B_j)$. On the other hand, the row space of $C$ is spanned by the row spaces of all $\{B_j\}_{j\in[1,n]}$, so $\text{rank}(C)$ cannot exceed the sum of the ranks of these blocks: $\Sigma_{j=1}^n\text{rank}(B_j)$. Therefore, a loose bound for $\text{rank}(C)$ is:
\[
\max_{1\leq j\leq n} \text{rank}(B_j)\leq \text{rank}(C) \leq \Sigma_{j=1}^n\text{rank}(B_j)
\]

Next, we present a method for exact rank calculation. Since the row space of $C$ is spanned by the row spaces of all $\{B_j\}_{j\in[1,n]}$, its rank is the sum of the ranks of the individual blocks minus the dimension of any shared row spaces, to avoid double-counting. This yields
\begin{align*}
\text{rank}(C) = &\Sigma_{j=1}^n\text{rank}(B_j) - \Sigma_{1\leq j<k\leq n}\dim(R(B_j)\cap R(B_k))\\
&+\Sigma_{1\leq j<k<l \leq n}\dim(R(B_j)\cap R(B_k)\cap R(B_l))\\
&-\cdots + (-1)^{n-1}\dim(R(B_1)\cap R(B_2)\cap \cdots \cap R(B_n))\\
\end{align*}
where $\dim(\cdot)$ denotes the dimension of a space, and $R(B_j)$ represents the row space spanned by $B_j$.